\documentclass{article}
\usepackage{fullpage}
\usepackage{setspace}
\usepackage{color}
\usepackage{array}
\usepackage{graphicx}
\usepackage{balance}
\usepackage{epstopdf}
\usepackage{subfigure}
\usepackage{xspace}
\usepackage{booktabs}
\usepackage{multirow}
\usepackage{times}
\usepackage{url}
\usepackage{setspace}

\usepackage{amsmath,amssymb,bm,bbm}
\usepackage{amsthm}
\usepackage{mathtools}

\usepackage[ruled,vlined,linesnumbered]{algorithm2e}
\DontPrintSemicolon

\SetKwProg{fnDef}{def}{}{}
\SetKwComment{tcp}{$\rhd$\, }{}

\newtheorem{definition}{Definition}
\newtheorem{lemma}{Lemma}
\newtheorem{fact}{Fact}

\newtheorem{theorem}{Theorem}

\newcommand{\ctr}{\mathcal{C}}

\newcommand{\dc}{\mathcal{A}}

\newcommand{\prob}[1]{{\mathbb P}\left[ #1 \right]}

\newcommand{\empProb}[1]{\hat{P}\left[ #1 \right]}
\newcommand{\estProb}[1]{\tilde{P}\left[ #1 \right]}
\newcommand{\probSub}[2]{{\mathbb P}_{#1} \left[ #2 \right]}

\newcommand{\naivebayes}{\texttt{Na{\"i}ve Bayes}\xspace}
\newcommand{\bayesnet}{Bayesian Network\xspace}
\newcommand{\bayesnets}{Bayesian networks\xspace}
\newcommand{\graph}{\mathcal{G}\xspace}

\newcommand{\vecx}{\bm{x}}

\newcommand{\exact}{\textsc{ExactMLE}\xspace}
\newcommand{\baseline}{\textsc{Baseline}\xspace}
\newcommand{\uniform}{\textsc{Uniform}\xspace}
\newcommand{\nonuniform}{\textsc{NonUniform}\xspace}

\newcommand{\bigo}[1]{O \left( #1 \right)}
\newcommand{\parent}[1]{\mathsf{par}\left( #1 \right)}
\newcommand{\domain}{\operatorname{dom}}

\newcommand{\nondes}[1]{\operatorname{NonDescendants} \left( #1 \right)}

\newcommand{\variance}[1]{{\bf Var}\left [ #1 \right ]}

\newcommand{\expec}[1]{{\mathbb E}\left [ #1 \right ]}

\newcommand{\eps}{\epsilon}

\newcommand{\remove}[1]{}

\DeclareMathOperator*{\argmax}{arg\,max}
\newcommand{\para}[1]{ \medskip \noindent {\bf #1}}
\newcommand{\tpara}[1]{ \noindent {\bf #1}}

\title{Learning Graphical Models from a Distributed Stream}

\author{
{Yu Zhang{\small $~^{\#1}$}, Srikanta Tirthapura{\small $~^{\#2}$}, Graham Cormode{\small $~^{*}$} }%
\vspace{1.6mm}\\
\fontsize{10}{10}\selectfont\itshape
$^{\#}$\, Electrical and Computer Engineering Department, Iowa State University\\
\fontsize{9}{9}\selectfont\ttfamily\upshape
$^{1}$\,yuz1988@iastate.edu
$^{2}$\,snt@iastate.edu%
\vspace{1.2mm}\\
\fontsize{10}{10}\selectfont\rmfamily\itshape
$^{*}$\,University of Warwick, 
\fontsize{9}{9}\selectfont\ttfamily\upshape
g.cormode@warwick.ac.uk
}

\begin{document}
\maketitle

\begin{abstract}
A current challenge for data management systems is to support the
construction and maintenance of machine learning models over data that
is large, multi-dimensional, and evolving.
While systems that could support these tasks are emerging, the need to scale
to distributed, streaming data requires new models and algorithms.
In this setting, as well as computational scalability and model
accuracy, we also need to minimize the amount of communication between
distributed processors, which is the chief component of latency.

We study Bayesian networks, the workhorse of graphical models, 
and present a communication-efficient method for continuously learning 
and maintaining a Bayesian network model over data that is 
arriving as a distributed stream partitioned across multiple processors. 
We show a strategy for maintaining model parameters 
that leads to an exponential reduction in communication when 
compared with baseline approaches to maintain the exact MLE (maximum
likelihood estimation).
Meanwhile, our strategy provides similar prediction errors for the
target distribution and for classification tasks. 
\end{abstract}

\section{Introduction}
\label{sec:introduction}
With the increasing need for large scale data analysis, distributed
machine learning~\cite{BBL11} has grown in importance in recent years.
Many platforms for distributed machine learning such as Tensorflow~\cite{tensorflow},
Spark MLlib~\cite{MBJ+16}, Petuum~\cite{XHD+15}, and 
Graphlab~\cite{LBG+12} have become popular in practice.
The raw data is described by a large number of interrelated variables,
and an important task is to describe the joint distribution over these variables, 
allowing inferences and predictions to be made.
For example, consider a large-scale sensor network where each sensor
is observing events in its local area (say, vehicles across a highway
network; or pollution levels within a city).
There can be many factors associated with each
event, such as duration, scale, surrounding environmental conditions,
and many other features collected by the sensor. 
However, directly modeling the full joint distribution of all these
features is infeasible, since 
the complexity of such a model grows exponentially with the number of variables.  
For instance, the complexity of a model with $n$ variables, each
taking one of $J$ values is $O(J^n)$ parameters.
The most common way to tame this complexity is to use a graphical model 
that can compactly encode the conditional dependencies among variables
in the data, and so reduce the number of parameters.

While many different graphical models have been proposed, we focus on
the most general and widely used class: {\em Bayesian networks}. 
A Bayesian network can be represented as a directed acyclic graph (DAG), 
where a node represents a variable and an edge directed from one node to 
another represents a conditional dependency between the corresponding variables. Bayesian networks 
have found applications in numerous domains, such as
decision making~\cite{WD14, ATA16, MB14} and cybersecurity~\cite{XLO+10, OAA16}.

While a graphical model can help in reducing the complexity, the number of 
parameters in such a model can still be quite high, and tracking each
parameter independently is expensive, especially in a distributed
system that sends a message for each update.
The key insight in our work is that it is not necessary to log every
event in real time; rather, we can aggregate information, and only
update the model when the new information causes a substantial change in the inferred model.
This still allows us to continuously maintain the model, but with
substantially reduced communication.
In order to give strong approximation guarantees for this approach,
we delve deeper into the construction of \bayesnets.

The fundamental task in building a \bayesnet is to estimate the
conditional probability distribution (CPD) of a variable given the
values assigned to its parents.
Once the CPDs of different variables are
known, the joint distribution can be derived over any subset of
variables using the chain rule \cite{KF09}.
To estimate the CPDs from empirical data, we use the maximum
likelihood estimation (MLE) principle.
The CPD of each event can be obtained by the ratio 
of the prevalence of that event versus the parent event (for 
independent variables, we obtain the single variable distribution).
Thus the central task is to obtain accurate counts of different
subsets of events. 

Our work is concerned with efficiently learning the parameters for a
given network structure.
Following the above discussion, the problem has a tantalizingly
clear central task: to materialize the needed CPDs using the observed frequencies in
the data. 
However, modern data analysis systems deal with massive, dynamic and
distributed data sources, such as network traffic monitors and 
large-scale sensor networks.
The raw volume of observations can be very large, 
and the simple solution of centralizing data would incur a very high communication
cost which is inefficient and infeasible.
Thus our key technical challenge is to design a scheme that can
accurately track a collection of distributed counts in a
communication-efficient way while guaranteeing the accuracy of the
current approximate model. 

In order to formalize the problem, we describe it using the 
continuous distributed stream monitoring model~\cite{CMY08}.
In this setting there are many sites, each
receiving an individual stream of observations (i.e. we assume
the data is horizontally partitioned).
A separate coordinator node, 
which receives no input itself, interacts with the sites to 
collaboratively monitor the union of the streams so far, and also answers
queries posed on the union of the streams so far.
This challenging model captures many of the difficulties that arise in
learning tasks in big data systems -- data is large, streaming in, and distributed over many
sites; and models need to be maintained in a timely manner allowing
for real-time responses.

Our work makes extensive use of a primitive called a {\em distributed counter}.
This allows us to count events accurately, without triggering a
message for each event.
We first show a basic monitoring scheme that uses distributed counters
independently for each variable in the model.
However, our strongest results arise when we provide a deeper
technical analysis of how the counts combine, to give tighter accuracy
guarantees with lower communication cost. 
The resulting exponential improvements in the worst-case cost for this
task are matched by dramatic reductions observed in practice. 
In more detail, our contributions are as follows: 

\para{Contributions.}
We present the first communication-efficient algorithms that continuously maintain 
a graphical model over distributed data streams. 

\noindent ---
Our algorithms maintain an accurate approximation of the Maximum
  Likelihood Estimate (MLE) using communication cost that is only
  {\em logarithmic} in the number of distributed observations.
  This is in contrast with the approach that maintains an exact MLE
  using a communication cost {\em linear} in the number of observations.

\noindent ---
  Our communication-efficient algorithms provide a provable
  guarantee that the model maintained is ``close'' to the MLE model
  given current observations, in a precise sense (Sections~\ref{sec:prelim},~\ref{sec:algo}).

\noindent ---
  We present three algorithms, in increasing order of efficiency
  and ability to capture model parameters, \baseline, \uniform, and
  \nonuniform in Section~\ref{sec:algo}.
  Our most general and communication-efficient algorithm, \nonuniform, is
  able to optimize communication cost for the case when the sizes of
  the CPDs of different random variables may be very different from
  each other.
  We also show how these algorithms apply to typical machine learning
  tasks such as classification (Section~\ref{sec:corollary}). 

\noindent ---
  We present an experimental evaluation in
  Section~\ref{sec:experiment}, showing that on a stream
  of a few million distributed training examples, our methods resulted in
  an improvement of 100-1000x in communication cost over the
  maintenance of exact MLEs, while providing estimates of joint
  probability with nearly the same accuracy as obtained by exact
  MLEs.

This provides a method for communication-efficient maintenance of
a graphical model over distributed, streaming data.
Prior works on maintaining a graphical model have considered efficiency in terms of space
(memory) and time, but these costs tend to be secondary when compared
to the communication cost in a distributed system.
Our method is built on the careful combination of multiple technical
pieces. 
Since the overall joint distribution is formed by
composing many CPDs, 
we divide the maximum ``error budget'' 
among the different parameters within the different CPDs
so that (a)~the error of the joint distribution is within the
desired budget, and (b)~the communication cost is as small as
possible.
We pose this as a convex optimization problem and use its
solution to parameterize the algorithms for distributed counters.
The next advance is to leverage concentration bounds
to argue that the aggregate behavior of the approximate
model consisting of multiple random variables (each estimating a
parameter of a CPD) is concentrated within a small range.
As a result, the dependence of the communication cost on
the number of variables $n$ can be brought down from $O(n)$ to $O(\sqrt{n})$.

\section{Prior and Related Work}
Many recent works are devoted to designing algorithms with efficient
communication in distributed machine learning.
Balcan {\em et al.}~\cite{BBF+12} were perhaps the first to give formal
consideration to this problem, based on the model of PAC (Probably
Approximately Correct) learning.
They showed lower bounds and algorithms for the non-streaming case,
where $k$ parties each hold parts of the input, and want to
collaborate to compute a model.
We call this ``the static distributed model''. 
Daum\'{e} {\em et al.}~\cite{DPS+12} considered a distributed version of the
classification problem: training data points are assigned labels, and
the goal is to build a model to predict labels for new examples.
Algorithms are also proposed in the static distributed model, where the
classifiers are linear separators (hyperplanes) allowing either no or
small error. 
Most recently, Chen {\em et al.} \cite{CSW+16} considered spectral graph clustering, and showed
that the trivial approach of centralizing all data can only be 
beaten when a broadcast model of communication is allowed.

In the direction of lower bounds, 
Zhang {\em et al.} \cite{ZDJ+13} considered the
computation of statistical estimators in the static distributed model, and show communication 
lower bounds for minimizing the expected squared error, based on
information theory. Phillips {\em et al.} \cite{PVZ12} show lower bounds using 
communication complexity arguments via the ``number in hand'' model.
Various functions related to machine
learning models are shown to be ``hard'' i.e., require large
amounts of communication in the distributed model .

Some previous works have extended sketching techniques to the problem
of streaming estimation of parameters of a Bayesian network. McGregor
and Vu~\cite{MV15} gave sketch-based algorithms to measure whether
given data was ``consistent'' with a prescribed model i.e. they
compare the empirical probabilities in the full joint distribution
with those that arise from fitting the same data into a particular
Bayesian network.  They also provide a streaming algorithm that finds
a good degree-one Bayesian network (i.e. when the graph is a tree).
Kveton {\em et al.}~\cite{KBG+16} adapt sketches to allow
estimation of parameters for models that have very high-cardinality
variables. However, neither of these methods consider the
distributed setting.

The continuous distributed monitoring model has been well studied in
the data management and algorithms communities, but there has been
limited work on machine learning problems in this model. A survey of
the model and basic results is given in~\cite{Cormode:11}. Efficient
distributed counting is one of the first problems studied in this
model~\cite{Dilman:Raz:01}, and subsequently
refined~\cite{Keralapura:Cormode:Ramimirtham:06,CMY08}.  The strongest
theoretical results on this problem are randomized algorithms due to
Huang {\em et al.}~\cite{HYZ12}.  Generic techniques are introduced
and studied by Sharfman {\em et al.}~\cite{Sharfman:Schuster:Keren:06,
Sharfman:Schuster:Keren:08}.
Some problems studied in this model include
clustering~\cite{Cormode:Muthukrishnan:Zhuang:07}, anomaly
detection~\cite{Huang:Nguyen:Garofalakis:Hellerstein:Jordan:Joseph:Taft:07},
entropy computation~\cite{ABC09} and sampling~\cite{CTW16}.

\section{Preliminaries}
\label{sec:prelim}
\newcommand{\distctr}{{\tt DistCounter}}
\newcommand{\epsfnA}{{\tt epsfnA}}
\newcommand{\epsfnB}{{\tt epsfnB}}

Let $\prob{E}$ denote the probability of event $E$. For random
variable $X$, let $\domain(X)$ denote the domain of $X$. 
We use $\prob{x}$ as a shorthand for $\prob{X=x}$ when the random
variable is clear from the context. 
For  a set of random variables $\mathcal{X} = \{X_1, \ldots, X_n\}$
let $\prob{X_1, \ldots, X_n}$ or $\prob{\mathcal{X}}$ 
denote the joint distribution over $\mathcal{X}$.  
Let $\domain(\mathcal{X})$ 
denote the set of all possible assignments to $\mathcal{X}$.


\begin{definition}
A {Bayesian network} $\graph=(\mathcal{X},\mathcal{E})$ is a
directed acyclic graph with a set of nodes 
$\mathcal{X} = \{X_1,\ldots, X_n \}$ and edges $\mathcal{E}$. 
Each $X_i$ represents a random variable. 
For $i \in [1,n]$, let $\parent{X_i}$ denote the set
of parents of $X_i$ and $\nondes{X_i}$ denote the variables that are not descendants of
$X_i$. The random variables obey the following condition: for each $i
\in [1,n]$, $X_i$ is conditionally independent of $\nondes{X_i}$,
given $\parent{X_i}$.
\end{definition}
For $i=1 \ldots n$, let $J_i$ denote the size of $\domain(X_i)$ and
$K_i$ the size of $\domain(\parent{X_i})$.

\para{Conditional Probability Distribution.}
Given a \bayesnet on $\mathcal{X}$, the joint distribution 
can be factorized as:
\begin{equation}
\label{eq:factorization}
\textstyle \prob{\mathcal{X}} = \prod_{i=1}^n \prob{X_i \mid \parent{X_i}}
\end{equation}
For each $i$,  $\prob{X_i \mid \parent{X_i}}$ is called
the {\em conditional probability distribution (CPD)} of $X_i$.  Let $\theta_i$
denote the CPD of $X_i$ and 
$\bm{\theta} = \{\theta_1,\ldots,\theta_n\}$ 
the set of CPDs of all variables.


\medskip
Given training data $\mathcal{D}$, we are interested in obtaining the {\bf maximum likelihood estimate (MLE) of $\bm{\theta}$.}
Suppose that $\mathcal{D}$ contains $m$ instances  $\xi[1], \ldots, \xi[m]$. 
Let $L(\bm{\theta} \mid \mathcal{D})$, the likelihood function of $\bm{\theta}$ given the dataset $\mathcal{D}$, 
be equal to the probability for dataset observed given those parameters.
\[
\textstyle
L(\bm{\theta} \mid \mathcal{D}) = \prob{\mathcal{D} \mid \bm{\theta}} 
\]

Let $L_i(\theta_i \mid \mathcal{D})$ denote the likelihood function for $\theta_i$. The likelihood function of $\bm{\theta}$ can be decomposed as a product of independent local likelihood functions.
\[
\textstyle
L(\bm{\theta} \mid \mathcal{D}) = \prod_{i=1}^n L_i(\theta_i \mid \mathcal{D})
\]
Let $\hat{\bm{\theta}}$ denote the value of $\bm{\theta}$ that maximizes the likelihood function, $\hat{\bm{\theta}}$ is also known as the Maximum Likelihoood Estimation (MLE). Similarly, let $\hat{\theta}_i$ denote the value of $\theta_i$ that maximizes $ L_i(\theta_i \mid \mathcal{D})$.

%
%

\begin{lemma}[\protect{\cite[proposition $17.1$]{KF09}}] 
\label{lem:mle-1}
Consider a \bayesnet with given structure $\graph$ and training dataset $\mathcal{D}$. Suppose for all $i \neq j$, $\theta_i$ and
$\theta_j$ are independent. For each $i \in [1,n]$, if $\hat{\theta}_i$ maximizes the likelihood
function
$L_i(\theta_i : \mathcal{D})$, then $\hat{\bm{\theta}} = 
\{\hat{\theta}_1, \ldots, \hat{\theta}_n \}$ maximizes $L(\bm{\theta}: \mathcal{D})$.
\end{lemma}

\para{Local CPD Estimation.}
In this work, we consider categorical random variables, so that the CPD of each variable $X_i$ can be represented as a table,
each entry is the probability $\probSub{i}{x_i \mid \vecx_i^{par}}$ where 
$x_i$ is the value of $X_i$ and $x_i \in \domain(X_i)$, 
$\vecx_i^{par}$ is the vector of values on the dimensions corresponding to $\parent{X_i}$ and
$\vecx_i^{par} \in \domain(\parent{X_i})$. 

We can handle continuous valued variables by appropriate
discretization, for example through applying a histogram, with bucket
boundaries determined by domain knowledge, or found by estimation on a random
sample. 


\begin{lemma}[\protect{\cite[Section~$17.2.3$]{KF09}}]
\label{lem:mle-2}
Given a training dataset $\mathcal{D}$, the maximum likelihood estimation (MLE) for $\theta_i$ is 
$\hat{\theta}_i(x_i \mid \vecx_i^{par}) = \frac{F_i(x_i, \vecx_i^{par})}{F_i(\vecx_i^{par})}$
where $F_i(x_i, \vecx_i^{par})$ is the number of events $(X_i=x_i, \parent{X_i}=\vecx_i^{par})$ 
in $\mathcal{D}$ and $F_i(\vecx_i^{par})$ is the number of events 
$(\parent{X_i} = \vecx_i^{par})$ in $\mathcal{D}$.
\end{lemma}

From Lemma~\ref{lem:mle-1}, a solution that maximizes the local likelihood functions also maximizes the joint likelihood function.
We further have that the MLE is an accurate estimate of the ground truth when the training dataset is sufficiently large.
\begin{lemma}[\protect{\cite[Corollary $17.3$]{KF09}}]
\label{lem:error-vs-m}
Given a \bayesnet $\graph$ on $\mathcal{X}$, let $P^*$ denote the ground truth joint distribution consistent with $\graph$ and $\hat{P}$ the joint distribution using MLE.
Suppose $\probSub{i}{x_i \mid \vecx_i^{par}} \geq \lambda$ for all $i, x_i, \vecx_i^{par}$. 
If $\textstyle m \geq \frac{1}{2 \lambda^{2(d+1)}} \frac{(1+\eps)^2}{\eps^2} \log \frac{nJ^{d+1}}{\delta}$
then
$\prob{ e^{-n \eps} \leq \frac{\hat{P}}{P^*} \leq e^{n \eps}}\!>1\!-\!\delta$,
where $J = \max_{i=1}^n J_i$ 
and $d$ the maximum number of parents for a variable in $\graph$.
\end{lemma}


\para{Approximate Distributed Counters.}
We make use of a randomized algorithm to continuously track 
counter values in the distributed monitoring model, due to~\cite{HYZ12}.

\begin{lemma}[\cite{HYZ12}]
\label{lem:rctr}
Consider a distributed system with $k$ sites. 
Given $0 < \eps < 1$, for $k \leq \frac{1}{\eps^2}$, there is a randomized distributed
algorithm $\textsc{DistCounter}\left( \eps, \delta \right)$ that continuously
maintains a distributed counter $\dc$ with the property that
$\expec{\dc} = \ctr$ and $\variance{\dc} \le (\eps \ctr)^2$, where
$\ctr$ is the exact value being counted.  
The communication cost is $O \left( \frac{\sqrt{k}}{\eps} \cdot \log T
\right)$ messages, where $T$ is the maximum value of $\ctr$.
The algorithm uses $O(\log T)$ space at each site and $O(1)$ amortized processing time per instance received.   
\end{lemma}

\para{Our Objective: Approximation to the MLE.}
Given a continuously changing data stream, exact maintenance of the MLE of the joint distribution is expensive communication-wise, since it requires the exact maintenance of multiple distributed counters, each of which may be incremented by many distributed processors. Hence, we consider the following notion of approximation to the MLE.

\begin{definition}
\label{def:approxmle}
Consider a \bayesnet $\graph$ on $\mathcal{X}$.
Let $\empProb{\cdot}$ denote the MLE of the joint distribution of
$\mathcal{X}$. Given approximation factor $0 < \eps < 1$, an
$\eps$-approximation to the MLE is a joint probability distribution
$\estProb{\cdot}$ such that, for any assignment of values $\bm{x}$ to $\mathcal{X}$,
$
e^{-\eps} \leq  \frac{\tilde{P}(\bm{x})}{\hat{P}(\bm{x})} \leq e^{\eps}
$.
Given an additional parameter $0 < \delta < 1$, a distribution $\tilde{P}$ is an $(\eps,\delta)$-approximation to MLE if it is an $\eps$-approximation to 
the MLE with probability at least $1-\delta$. 
\end{definition}

{Our goal is to maintain a distribution $\tilde{P}$ that is an
$(\eps,\delta)$-approximation to the MLE, given all data observed so far, 
in the distributed continuous model.}

The task of choosing the graph $\graph$ with which to
model the data (i.e. which edges are present in the network and which are not) is also an
important one, but one that we treat as orthogonal to our focus in
this work.
For data of moderate dimensionality, we may assume that the
graph structure is provided by a domain expert, based on known
structure and independence within the data.
Otherwise, the graph
structure can be learned offline based on a suitable sample of the
data.
The question of learning graph models ``live'' as data arrives,
is a challenging one that we postpone to future work.

%

\section{Distributed Streaming MLE Approximation}
\label{sec:algo}
Continuous maintenance of the MLE requires continuous maintenance of a
number of counters, to track the different (empirical) conditional
probability distributions.

 
For each $x_i \in \domain(X_i)$ and $\vecx_i^{par} \in \domain(\parent{X_i})$,
let $\ctr_i(\vecx_i^{par})$ be the counter that tracks $F_i(\vecx_i^{par})$, and 
let $\ctr_i(x_i, \vecx_i^{par})$ be the counter that tracks $F_i(x_i, \vecx_i^{par})$. 
When clear from the context, we use the counter to also denote its value 
when queried. 
Consider any input vector $\vecx=\langle
x_1, \ldots, x_n \rangle$. For $1 \le i \le n$, let $\vecx_i^{par}$
denote the projection of vector $\vecx$ on the dimensions
corresponding to $\parent{X_i}$.
Based on Equation~\ref{eq:factorization} and Lemma~\ref{lem:mle-1},
the empirical joint probability $\empProb{\vecx}$ can be factorized as:
%
\begin{equation}
\label{eq:joint-ct2}
\textstyle
\empProb{\vecx} = \prod_{i=1}^n \frac{\ctr_i(x_i,\vecx_i^{par})}{\ctr_i(\vecx_i^{par})}
\end{equation} 

\subsection{Strawman: Using Exact Counters}
\label{sec:exact-ct}

A simple solution to maintain parameters is to maintain each
counter $\ctr_i(\cdot)$ and $\ctr_i(\cdot,\cdot)$ exactly at all
times, at the coordinator. With this approach, the coordinator always
has the MLE of the joint distribution, but the communication
cost quickly becomes the bottleneck of the whole system.
Each time an event is received at a site, the site
tells the coordinator to update the centralizing parameters $\theta$
immediately, essentially losing any benefit of distributed processing. 


\begin{lemma}
If exact counters are used to maintain the MLE of a Bayesian network on $n$ variables in the distributed monitoring model, the total communication cost to continuously maintain the model over $m$ event observations is $O(mn)$, spread across $m$ messages of size $n$.
\end{lemma}


\allowdisplaybreaks
\subsection{Master Algorithms Using Approximate Counters}
\label{sec:generic-approx}
\newcommand{\parproj}{par}
The major issue with using exact counters to maintain the MLE is
the communication cost, which increases linearly with the number of
events received from the stream. 
We describe a set of ``master'' algorithms that we use to
approximately track statistics,
leading to a reduced communication cost, yet maintaining an approximation of the MLE.
In successive sections we tune their parameters and analysis 
to improve their behavior.
In Section~\ref{sec:baseline}, we describe the \baseline algorithm
which divides the error budget uniformly and pessimistically across
all variables.
Section~\ref{sec:uniform} gives the \uniform approach, which keeps the
uniform allocation, but uses an improved randomized analysis.
Finally, the \nonuniform algorithm in Section~\ref{sec:nonuniform}
adjusts the error budget allocation to account for the cardinalities
of different variables. 

These algorithms build on top of 
approximate distributed counters (Lemma~\ref{lem:rctr}), denoted by $\dc$. 
At any point, the coordinator can answer a query
over the joint distribution by using the outputs of the approximate
counters, rather than the exact values of the counters (which it no
longer has access to).
We have the following objective:

\begin{definition}[MLE Tracking Problem]
\label{dfn:obj}
Given $0<\eps<1$, for $i \in [1,n]$, 
we seek to maintain  distributed counters $\dc_i(x_i, \vecx_i^{\parproj})$ and 
$\dc_i(\vecx_i^{\parproj})$ such that for any data input vector 
$\vecx = \langle x_1, x_2, \ldots, x_n \rangle$, we have
\[\textstyle e^{-\eps} \leq \frac{\tilde{P}(\bm{x})}{\hat{P}(\bm{x})} = \prod_{i=1}^n \left(\frac{\dc_i(x_i,\vecx_i^{\parproj})}{\ctr_i(x_i,\vecx_i^{\parproj})} \cdot \frac{\ctr_i(\bm{x_i})}{\dc_i(\vecx_i^{\parproj})} \right) \leq e^{\eps}\]
\end{definition}

Our general approach is as follows. 
Each algorithm initializes a set of distributed counters
(Algorithm~\ref{algo:generic-init}).
Once a new event is received,
we update the two counters associated with the CPD for each variable
(Algorithm~\ref{algo:generic-train}).
A query is processed as in Algorithm~\ref{algo:generic-query} by
probing the approximate CPDs.
The different algorithms are specified based on how they set the error
parameters for the distributed counters, captured in the functions
$\epsfnA$ and
$\epsfnB$.
\begin{algorithm}[t]
\label{algo:generic-init}
\caption{\textsc{Init}$(n,\epsfnA,\epsfnB)$}
\tcc{Initialization of Distributed Counters.} 
\KwIn{$n$ is the number of variables. $\epsfnA$ and $\epsfnB$ 
are parameter initialization functions provided by specific algorithms.
}
\ForEach{$i$ \textbf{from} $1$ \textbf{to} $n$} {
        \ForEach{$x_i \in \domain(X_i)$, $\vecx_i^{par} \in \domain(\parent{X_i})$}{$\dc_i(x_i, \vecx_i^{par}) \gets \distctr(\epsfnA(i), \delta)$}

        \ForEach{$\vecx_i^{par} \in \domain(\parent{X_i})$}{$\dc_i(\vecx_i^{par}) \gets \distctr(\epsfnB(i), \delta)$}
}
\end{algorithm}

\begin{algorithm}[t]
\label{algo:generic-train}
\caption{\textsc{Update}($\vecx$)}
\tcc{Called by a site upon receiving a new event}
\KwIn{$\vecx=\langle x_1, \ldots, x_d \rangle$ is an observation.}
\ForEach{$i$ \textbf{from} $1$ \textbf{to} $n$} {
  Increment $\dc_i(x_i,\vecx_i^{\parproj})$ \;
  Increment $\dc_i(\vecx_i^{\parproj})$  \;
}
\end{algorithm}

\begin{algorithm}[t]
\label{algo:generic-query}
\caption{\textsc{Query}($\vecx$)}
\tcc{Used to query the joint probability distribution.}
\KwIn{$\vecx = \langle x_1, \ldots, x_d \rangle$ is an input vector}
\KwOut{Estimated Probability $\estProb{\vecx}$}
\ForEach{$i$ \textbf{from} $1$ \textbf{to} $n$}
{
$p_i \gets \frac{\dc_i(x_i,\vecx_i^{\parproj})}{\dc_i(\vecx_i^{\parproj})}$\;
}
Return $\prod_{i=1}^n p_i$ \;
\end{algorithm}

\subsection{\baseline Algorithm Using Approximate Counters}
\label{sec:baseline}
Our first approach \baseline, sets the error
parameter of each counter $\dc(\cdot)$ and $\dc(\cdot,\cdot)$ to a
value $\frac{\eps}{3n}$, which is small enough so that the
overall error in estimating the MLE is within desired bounds.
In other words, \baseline
configures Algorithm~\ref{algo:generic-init} with
$\epsfnA(i) = \epsfnB(i) = \frac{\eps}{3n}$.
Our analysis makes use of the following standard fact. 
\begin{fact}
\label{lem:uniform-approx}
For $0< \eps <1$ and $n \in \mathbb{Z}^+$, when $\alpha \leq \frac{\eps}{3n}$
\[ \textstyle
\left( \frac{1+\alpha}{1-\alpha} \right)^n \leq e^{\eps} 
\quad \text{and} \quad
\left( \frac{1-\alpha}{1+\alpha} \right)^n \geq e^{-\eps}
\]
\end{fact}
\remove{
\begin{proof}
When $\alpha \leq \frac{\eps}{3n} \leq \frac{1}{3}$, we have that $\frac{2\alpha}{1-\alpha} \leq 3\alpha$,
\[
\left(\frac{1+\alpha}{1-\alpha}\right)^n  =  \left(1 + \frac{2\alpha}{1-\alpha}\right)^n  
 \leq \left(1 + 3\alpha \right)^n   \leq (e^{3\alpha})^n  \leq e^{\eps}
\]
So we have $\left( \frac{1+\alpha}{1-\alpha} \right)^n \leq e^{\eps}$;
taking the reciprocal of both sides proves the other direction. 
%
\end{proof}}

\begin{lemma}
\label{lem:uniform-lem}
Given $0<\eps,\delta<1$ and a Bayesian network with $n$ variables, the
\baseline algorithm maintains the parameters of the Bayesian network
such that at any point, it is an $(\eps,\delta)$-approximation to the
MLE. The total communication cost across $m$ training observations is 
$O\left(\frac{n^2 J^{d+1} \sqrt{k}}{\eps} \cdot \log \frac{1}{\delta} \cdot \log m \right)$ 
messages, where $J$ is the maximum domain cardinality for any variable $X_i$, 
$d$ is the maximum number of parents for 
a variable in the Bayesian network and $k$ is the number of sites.
\end{lemma}

\begin{proof} We analyze the ratio
\[ \textstyle
\frac{\tilde{P}(\vecx)}{\hat{P}(\vecx)} = \prod_{i=1}^n \frac{\dc_i(x_i,\vecx_i^{\parproj})} {\dc_i(\vecx_i^{\parproj})} \cdot \frac{\ctr_i(\vecx_i^{\parproj})}{\ctr_i(x_i,\vecx_i^{\parproj})}
\]
By rescaling the relative error and applying Chebyshev's inequality and the union 
bound to the approximate counters of Lemma~\ref{lem:rctr}, we have that each counter $\dc_i()$
is in the range $(1 \pm \frac{\eps}{3n}) \cdot \ctr_i()$ with probability at least $1-\delta$.
The worst case is when 
$\dc_i(x_i,\vecx_i^{\parproj}) = \left( 1-\frac{\eps}{3n} \right) \cdot \ctr_i(x_i,\vecx_i^{\parproj})$ and $\dc_i(\vecx_i^{\parproj}) = (1+\frac{\eps}{3n}) \cdot
\ctr_i(\vecx_i^{\parproj})$, i.e each counter takes on an extreme value within its confidence interval.
In this case, $\frac{\tilde{P}(\vecx)}{\hat{P}(\vecx)}$ takes on the minimum
value.  Using Fact~\ref{lem:uniform-approx}, we get
$
\frac{\tilde{P}(\vecx)}{\hat{P}(\vecx)} \geq \left( \frac{1-\frac{\eps}{3n}}{1+\frac{\eps}{3n}} \right)^n \geq e^{-\eps}
$.
Symmetrically, we have
$\frac{\tilde{P}(\vecx)}{\hat{P}(\vecx)} \leq e^{\eps}$ when we make
pessimistic assumptions in the other direction.

Using Lemma~\ref{lem:rctr}, the communication cost for each
distributed counter is $O \left( \frac{n \sqrt{k}}{\eps} \cdot \log
\frac{1}{\delta} \cdot \log m \right)$ messages. For each $i \in [1,n]$, 
there are at most $J^{d+1}$ counters $\dc_i(x_i, \vecx_i^{\parproj})$
and at most $J^d$ counters $\dc_i(\vecx_i^{\parproj})$ 
for all $x_i \in \domain(X_i)$ and $\vecx_i^{\parproj} \in \domain(\parent{X_i})$.
So the total communication cost is 
$O\left(\frac{n^2 J^{d+1} \sqrt{k}}{\eps} \cdot \log \frac{1}{\delta} \cdot \log m\right)$
messages.
\end{proof}

\subsection{\uniform: Improved Uniform Approximate Counters}
\label{sec:uniform}
The approach in \baseline is overly pessimistic: it
assumes that all errors may fall in precisely the worst possible
direction.
Since the counter algorithms are unbiased and random, we
can provide a more refined statistical analysis and still obtain our
desired guarantee with less communication.

Recall that the randomized counter algorithm in Lemma~\ref{lem:rctr} can be shown to have
the following properties:
\begin{itemize}
\item Each distributed counter is unbiased, $\expec{\dc} = \ctr$.
\item The variance of counter is bounded, $\variance{\dc} \le (\eps' \ctr)^2$, where $\eps'$ is the error
parameter used in $\dc$.
\end{itemize}

Hence the product of multiple distributed counters is
also unbiased, and we can also bound the variance of the product.

Our \uniform algorithm initializes its state using Algorithm~\ref{algo:generic-init}
with  $\epsfnA(i) = \epsfnB(i) = \frac{\eps}{16\sqrt{n}}$.
We prove its properties after first stating a useful fact. 

\begin{fact}
\label{fact:math-fact}
When $0<x<0.3$, $e^x < 1+2x$ and  $e^{-2x} < 1-x$.
\end{fact}

\remove{
\begin{proof}
Let $f(x)=e^x-2x-1$, we have 
\[
f'(x)=e^x-2<0 \quad \text{when} \quad 0<x<0.3
\] 
So $f(x)<f(0)=0$, $e^x<1+2x$. 
Similarly, we have $e^{-2x}<1-x$.
\end{proof}}

\begin{lemma}
\label{lem:loose-approx1}
Given input vector $\vecx = \langle x_1, \ldots, x_d \rangle$, 
let $F=\prod_{i=1}^n \dc_i(x_i,\vecx_i^{\parproj})$ 
and $f=\prod_{i=1}^n \ctr_i(x_i,\vecx_i^{\parproj})$.
With Algorithm \uniform, $\expec{F} = f$ and $\variance{F} \leq \frac{\eps^2}{128} \cdot f^2$. 
\end{lemma}

\begin{proof}
From Lemma~\ref{lem:rctr}, for $i \in [1,n]$ we have
\[
\expec{\dc_i(x_i,\vecx_i^{\parproj})} = \ctr_i(x_i,\vecx_i^{\parproj}).
\]
Since all the $\dc_i(\cdot,\cdot)$ variables are independent, we have:
\[
\expec{\prod_{i=1}^n \dc_i(x_i,\vecx_i^{\parproj})} = \prod_{i=1}^n \ctr_i(x_i,\vecx_i^{\parproj})
\] 
This proves $\expec{F} = f$.
We next compute $\expec{\dc_i^2(x_i,\vecx_i^{\parproj})}$,
\begin{align*}
\expec{\dc_i^2(x_i,\vecx_i^{\parproj})} & = \variance{\dc_i(x_i,\vecx_i^{\parproj})} + \left( \expec{\dc_i(x_i,\vecx_i^{\parproj})} \right)^2 \\
& \leq \left( \epsfnA(i) \cdot \ctr_i(x_i,\vecx_i^{\parproj}) \right)^2 + \ctr_i^2(x_i,\vecx_i^{\parproj})   \\
& \leq \left( 1+\frac{\eps^2}{256n}  \right) \cdot \ctr_i^2(x_i,\vecx_i^{\parproj})
\end{align*}

By noting that different terms $\dc_i^2(x_i,\vecx_i^{\parproj})$ are independent:
\begin{align*}
\expec{F^2} & =  \expec{\Big( \prod_{i=1}^n \dc_i(x_i,\vecx_i^{\parproj}) \Big)^2} 
 = \prod_{i=1}^n \expec{\dc_i^2 (x_i,\vecx_i^{\parproj})} \\
& \leq \Big( 1+\frac{\eps^2}{256n}  \Big)^n \cdot \prod_{i=1}^n \ctr_i^2(x_i,\vecx_i^{\parproj}) 
  \leq e^{\eps^2 / 256} \cdot f^2
\end{align*}
%
\[
\text{Using Fact~\ref{fact:math-fact},~}
\expec{F^2} \leq e^{\eps^2 / 256} \cdot f^2 \leq \left( 1+\frac{\eps^2}{128} \right) \cdot f^2
\]
Since $\expec{F} = f$, we calculate $\variance{F}$:
\[
\variance{F} = \expec{F^2} - \left( \expec{F} \right)^2 
 \leq \left( 1 + \frac{\eps^2}{128} \right) \cdot f^2 - f^2
 = \frac{\eps^2}{128} \cdot f^2
\]
\end{proof}
\remove{
\begin{align*}
\variance{F} & = \expec{F^2} - \left( \expec{F} \right)^2  \\
& \leq \left( 1 + \frac{\eps^2}{128} \right) \cdot f^2 - f^2 \\
& = \frac{\eps^2}{128} \cdot f^2
\end{align*}
}

Using Chebyshev's inequality, we can bound $F$.
\begin{lemma}
\label{lem:loose-approx2}
For $i \in [1,n]$, maintaining distributed counters $\dc_i(x_i,\vecx_i^{\parproj})$ with 
approximation factor $\frac{\eps}{16\sqrt{n}}$,
gives
$
e^{-\frac{\eps}{2}} \leq \prod_{i=1}^n \frac{\dc_i(x_i,\vecx_i^{\parproj})}{\ctr_i(x_i,\vecx_i^{\parproj})} \leq e^{\frac{\eps}{2}}
$
with  probability at least $7/8$.
\end{lemma}

\begin{proof}
Using the Chebyshev inequality, with $\expec{F}=f$ 
\[
\textstyle
\prob{|F-f| \leq \sqrt{8} \cdot \sqrt{\variance{F}}} \geq \frac{7}{8}
\]
From Lemma~\ref{lem:loose-approx1}, $\variance{F} \leq \frac{\eps^2}{128} \cdot f$, hence
\[
\textstyle
\prob{|F-f| \leq \frac{\eps f}{4}} \geq \frac{7}{8}
\]
\[
\text{and so (via Fact~\ref{fact:math-fact}),~}
e^{-\frac\eps{2}} \leq \left(1-\frac{\eps}{4}\right) \leq \frac{F}{f} \leq \left(1+\frac{\eps}{4}\right) \leq e^{\frac{\eps}{2}}
\]
with probability at least $7/8$.
\end{proof}

For the term $\frac{\ctr_i(\vecx_i^{\parproj})}{\dc_i(\vecx_i^{\parproj})}$, we maintain
distributed counters $\dc_i(\vecx_i^{\parproj})$ with approximation factor
$\frac{\eps}{16 \sqrt{n}}$. 
One subtlety here is that different variables, say $X_i$ and $X_j$, $i \neq j$ can have 
$\parent{X_i} = \parent{X_j}$, so that 
$\prod_{i=1}^n \frac{\ctr_i(\vecx_i^{\parproj})}{\dc_i(\vecx_i^{\parproj})}$ can have duplicate terms, arising from 
different $i$.
This leads to terms in the product that are not independent of each other.
To simplify such cases, for each $i \in [1,n]$, 
we maintain separate distributed counters $\dc_i(\vecx_i^{\parproj})$,
so that when $\parent{X_i} = \parent{X_j}$, the counters $\dc_i(\vecx_i^{\parproj})$ and $\dc_j(\vecx_j^{\parproj})$
are independent of each other. 
Then, we can show the following lemma for counters
$\dc(\vecx_i^{\parproj})$, which is derived in a manner similar to Lemma~\ref{lem:loose-approx1}
and~\ref{lem:loose-approx2}. 

\begin{lemma}
\label{lem:loose-approx3}
For $i \in [1,n]$, when we maintain distributed counters 
$\dc_i(\vecx_i^{\parproj})$ with approximation factor $\frac{\eps}{16\sqrt{n}}$, we have 
$e^{-\frac{\eps}{2}} \leq \prod_{i=1}^n \frac{\ctr_i(\vecx_i^{\parproj})}{\dc_i(\vecx_i^{\parproj})} \leq e^{\frac{\eps}{2}}$
with probability at least $7/8$.
\end{lemma}

Combining these results, we obtain the following result about \uniform.
\begin{theorem}
\label{thm:uniform}
Given $0<\eps,\delta<1$, \uniform algorithm continuously maintains an
$(\eps,\delta)$-approximation to the MLE over the course of $m$
observations. The communication cost over all observations is
$O\left(\frac{n^{3/2} J^{d+1} \sqrt{k}}{\eps} \cdot \log
\frac{1}{\delta} \cdot \log m \right)$ messages, where $J$ is the
maximum domain cardinality for any variable $X_i$, $d$ is the maximum
number of parents for a variable in the Bayesian network, and $k$ is
the number of sites.
\end{theorem}

\begin{proof} Recall that our approximation ratio is given by
\[
\frac{\tilde{P}(\vecx)}{\hat{P}(\vecx)} = \prod_{i=1}^n \frac{\dc_i(x_i,\vecx_i^{\parproj})} {\dc_i(\vecx_i^{\parproj})} \cdot \frac{\ctr_i(\vecx_i^{\parproj})}{\ctr_i(x_i,\vecx_i^{\parproj})}
\]

Combining Lemmas~\ref{lem:loose-approx2} and~\ref{lem:loose-approx3}, we have
\[
e^{-\eps} \leq \prod_{i=1}^n \frac{\dc_i(x_i,\vecx_i^{\parproj})}{\ctr_i(x_i,\vecx_i^{\parproj})} \cdot \frac{\ctr_i(\vecx_i^{\parproj})}{\dc_i(\vecx_i^{\parproj})} \leq e^{\eps} 
\]
with probability at least $3/4$, showing that the model that is maintained is an $(\eps,1/4)$ approximation to the MLE.
By taking the median of $O(\log \frac{1}{\delta})$ independent instances of the \uniform algorithm, we improve the error probability to $\delta$.

The communication cost for each distributed counter is 
$\bigo{\frac{\sqrt{n k}}{\eps} \cdot \log \frac{1}{\delta} \cdot \log m}$ messages.
For each $i \in [1,n]$, there are at most $J^{d+1}$ counters
$\dc_i(x_i,\vecx_i^{\parproj})$ for all $x_i \in \domain(X_i)$ and
$\vecx_i^{\parproj} \in \domain(\parent{X_i})$, and at most $J^d$
counters $\dc_i(\vecx_i^{\parproj})$ for all $\vecx_i^{\parproj} \in
\domain(\parent{X_i})$.
So the total communication cost is
$\bigo{\frac{n^{3/2} J^{d+1} \sqrt{k}}{\eps} \cdot \log \frac{1}{\delta} \cdot \log m }$ messages.
\end{proof}

\subsection{Non-uniform Approximate Counters}
\label{sec:nonuniform}
In computing the communication cost of \uniform, we made
the simplifying assumption that the domains of different variables are
of the same size $J$, and each variable has the same number of
parents $d$ \footnote{Note that these assumptions were only used to
determine the communication cost, and do not affect the correctness
of the algorithm.}.
While this streamlines the analysis,
it misses a chance to more tightly bound the communication by better
adapting to the cost of parameter estimation.
Our third algorithm, \nonuniform, has a more involved analysis by
making more use of the information about the \bayesnet. 

We set the approximation parameters of distributed counters
$\dc_i(x_i, \vecx_i^{par})$ and $\dc_i(\vecx_i^{par})$ as a function of the values
$J_i$ (the cardinality of $\domain(X_i)$) and $K_i$ (the cardinality
of $\domain\left(\parent{X_i} \right)$).
To find the settings that yield the best tradeoffs, we express the
total communication cost as a function of different $J_i$s and
$K_i$s. Consider first the maintenance of the CPD for variable $X_i$,
this uses counters of the form $\dc_i(\cdot, \cdot)$. Using an
approximation error of $\nu_i$ for these counters leads to a
communication cost proportional to $\frac{J_i K_i}{\nu_i}$, since the
number of such counters needed at $X_i$ is $J_i K_i$.  Thus, the total
cost across all variables is $\sum_{i=1}^n \frac{J_i K_i}{\nu_i}$.
In order to ensure correctness (approximation to the MLE), we consider
the variance of our estimate of the joint probability distribution.
Let $F=\prod_{i=1}^n \dc_i(x_i,\vecx_i^{\parproj})$ and 
$f=\prod_{i=1}^n \ctr_i(x_i,\vecx_i^{\parproj})$.
%
\begin{align}
  \label{nu:cond1}
  \textstyle
  \expec{F^2}
  &
  \textstyle
  = \prod_{i=1}^n \left(1+\nu_i^2 \right) \cdot f^2 \leq \prod_{i=1}^n
  e^{\nu_i^2} \cdot f^2   \nonumber \\
  &
    \textstyle
  = e^{\left(\sum_{i=1}^n \nu_i^2\right)} \cdot f^2 \leq \big(1+2\sum_{i=1}^n \nu_i^2 \big) \cdot f^2
\end{align}
%
From Lemma~\ref{lem:loose-approx1}, to bound the error of the joint
distribution, we want that 
$\expec{F^2} \leq \big(1+\frac{\eps^2}{128} \big) \cdot f^2$ which can be ensured provided
that the following condition is satisfied,
%
\begin{equation}
  \label{nu:cond2}
  \textstyle
\sum_{i=1}^n \nu_i^2 \leq \eps^2/256
\end{equation}
%
Thus, the problem is to find values of $\nu_1, \ldots, \nu_n$ to
minimize communication while satisfying this
constraint.
That is, 
\begin{equation}
\label{eq:non-uniform}
\text{Minimize} \quad  \sum_{i=1}^n \frac{J_iK_i}{\nu_i}   
\quad \text{subject to} \quad \sum_{i=1}^n \nu_i^2 = \frac{\eps^2}{256}
\end{equation}

Using the Lagrange Multiplier Method, 
let $\mathcal{L} = \sum_{i=1}^n \frac{J_iK_i}{\nu_i} + \lambda \left( \nu_i^2 - \frac{\eps^2}{256}  \right)$, 
we must satisfy:
\begin{equation}
\label{eq:lagrange-eq}
\left\{
  \begin{array}{l}
  \frac{\partial \mathcal{L}}{\partial \nu_1} = -\frac{J_1K_1}{\nu_1^2} + 2\lambda\nu_1 = 0 \\
  \frac{\partial \mathcal{L}}{\partial \nu_2} = -\frac{J_2K_2}{\nu_2^2} + 2\lambda\nu_2 = 0 \\  
  \quad \quad \vdots  \\
  \frac{\partial \mathcal{L}}{\partial \nu_n} = -\frac{J_nK_n}{\nu_n^2} + 2\lambda\nu_n = 0 \\ 
  \sum_{i=1}^n \nu_i^2 = \frac{\eps^2}{256}\\
  \end{array}
\right.
\end{equation}

Solving the above equations, the optimal parameters are:
\begin{equation}
  \label{eq:non-uniform-sol1}
  \textstyle
\nu_i = \frac{(J_i K_i)^{1/3} \eps}{16 \alpha}, \quad \text{where} \quad \alpha =  \Big( \sum_{i=1}^n (J_i K_i)^{2/3} \Big)^{1/2}
\end{equation}

Next we consider the distributed counters $\dc(\cdot)$. For each $i
\in [1,n]$ and each $\vecx_i^{\parproj} \in \domain(\parent{X_i})$, 
we maintain $\dc_i(\vecx_i^{\parproj})$ independently and ignore the
shared parents as we did in the Section~\ref{sec:uniform}. Let $\mu_i$ denote the
approximation factor for $\dc_i(\vecx_i^{\parproj})$, the communication cost for
counter $\dc_i(\vecx_i^{\parproj})$ is proportional to $\sum_{i=1}^n
\frac{K_i}{\mu_i}$ and the restriction due to bounding the error of
joint distribution is $\sum_{i=1}^n \mu_i^2 \leq \frac{\eps^2}{256}$. Similarly to above, 
the solution via the Lagrange multiplier method is
%
\begin{equation}
\label{eq:non-uniform-sol2}
\mu_i = \frac{K_i^{1/3} \eps}{16 \beta}, \quad \text{where} \quad \beta = \Big( \sum_{i=1}^n K_i^{2/3} \Big)^{1/2}
\end{equation} 
\noindent
Setting
$\epsfnA(i)=\nu_i$ as in~\eqref{eq:non-uniform-sol1} and
$\epsfnB(i)=\mu_i$ as in~\eqref{eq:non-uniform-sol2}
in Algorithm~\ref{algo:generic-init}
gives our \nonuniform algorithm.


\begin{theorem}
\label{thm:nonuniform}
Given $0<\eps,\delta<1$, \nonuniform continuously 
maintains an $(\eps,\delta)$-approximation to the MLE given $m$ training observations.
The communication cost over all observations is
$O\left( \Gamma \cdot \frac{\sqrt{k}}{\eps} \cdot \log \frac{1}{\delta} \cdot \log m  \right)$ messages,
where 
$\Gamma = \left( \sum_{i=1}^n (J_i K_i)^{2/3} \right)^{3/2}$   $+$  $\left( \sum_{i=1}^n K_i^{2/3} \right)^{3/2}$
\end{theorem}
\begin{proof}
Let $F=\prod_{i=1}^n \dc_i(x_i,\vecx_i^{\parproj})$ 
and $f=\prod_{i=1}^n \ctr_i(x_i,\vecx_i^{\parproj})$.
From Conditions~\ref{nu:cond1} and \ref{nu:cond2}, we bound the variance of $F$
\[
\variance{F} = \expec{F^2} - \left( \expec{F} \right)^2 \leq \left(1 + \frac{\eps^2}{128}\right) \cdot f^2 - f^2 = \frac{\eps^2}{128} \cdot f^2
\]
By Lemma~\ref{lem:loose-approx2}, with probability at least $7/8$ we have
\[
e^{-\frac\eps{2}} \leq \frac{F}{f} \leq  e^{\frac{\eps}{2}}
\] 
Thus, 
\[
e^{-\frac\eps{2}} \leq \frac{\prod_{i=1}^n \dc_i(x_i,\vecx_i^{\parproj})}{\prod_{i=1}^n \ctr_i(x_i,\vecx_i^{\parproj})} \leq  e^{\frac{\eps}{2}}
\]
Similarly for counter $\dc_i(\vecx_i^{\parproj})$, 
\[
e^{-\frac\eps{2}} \leq \frac{\prod_{i=1}^n \dc_i(\vecx_i^{\parproj})}{\prod_{i=1}^n \ctr_i(\vecx_i^{\parproj})} \leq  e^{\frac{\eps}{2}}
\]
Combining above two equations, we prove the correctness of \nonuniform : given input $\vecx$, we have
\[
e^{-\eps} \leq \frac{\tilde{P}(\vecx)}{\hat{P}(\vecx)} \leq e^{\eps}
\]

For each $i \in [1, n]$, $x_i \in \domain(X_i)$ and $\vecx_i^{\parproj} \in \domain(\parent{X_i})$, 
the communication cost to maintain the counter $\dc_i(x_i,\vecx_i^{\parproj})$ is 
$O\left( \frac{1}{\nu_i} \cdot \log \frac{1}{\delta} \cdot \log m \right)$. 
As $J_i$ is the cardinality of $\domain(X_i)$ and $K_i$ is the cardinality of $\domain(\parent{X_i})$, 
the communication cost for all $\dc_i(\cdot , \cdot)$ counters $M_1$ is
\[
M_1 = \sum_{i=1}^n \frac{J_i K_i \sqrt{k}}{\nu_i} \cdot \log \frac{1}{\delta} \cdot \log m
\] 
By substituting the values of $\nu_i$ in Equation~\ref{eq:non-uniform-sol1} to the expression of $M_1$, we obtain
\[
M_1 = \left( \sum_{i=1}^n (J_i K_i)^{2/3} \right)^{3/2} \cdot \frac{\sqrt{k}}{\eps} \cdot \log \frac{1}{\delta} \cdot \log m
\]
Similarly, the communication cost for all $\dc_i(\cdot)$ counters $M_2$ is
\[
M_2 = \left( \sum_{i=1}^n K_i^{2/3} \right)^{3/2} \cdot \frac{\sqrt{k}}{\eps} \cdot \log \frac{1}{\delta} \cdot \log m
\]
The total communication cost to maintain all the counters is 
$M_1 + M_2 = O\left( \Gamma \cdot \frac{\sqrt{k}}{\eps} \cdot \log \frac{1}{\delta} \cdot \log m  \right)$.
\end{proof}



\para{Comparison between \uniform and \nonuniform.}
Note that \uniform and \nonuniform have the same dependence on $k$,
$\eps, \delta$, and $m$. To compare the two algorithms, we focus on
their dependence on the $J_i$s and $K_i$s.  Consider a case when all
but one of the $n$ variables are binary valued, and variable $X_1$ can
take one of $J$ different values, for some $J \gg 1$.  Further,
suppose that (1)~the network was a tree so that $d$, the maximum
number of parents of a node is $1$, and (2)~$X_1$ was a leaf in the
tree, so that $K_i=1$ for all nodes $X_i$.  The communication bound
for \uniform by Theorem~\ref{thm:uniform} is $O(n^{1.5} J^2)$, while
the bound for \nonuniform by Theorem~\ref{thm:nonuniform} is
$O({(n+J^{2/3})}^{1.5}) = O(\max\{n^{1.5}, J\})$.  In this case,
our analysis argues that
\nonuniform provides a much smaller communication cost than \uniform.



\section{Special Cases and Extensions}
\label{sec:corollary}

Section~\ref{sec:algo} showed that \nonuniform
has the tightest bounds on communication cost to  maintain an
approximation to the MLE.
In this section, we apply \nonuniform to networks with special structure,
such as 
Tree-Structured Network and Na{\"i}ve Bayes, as well as to a
classification problem.


\begin{figure}[t]
\centering
\includegraphics[width=0.5\textwidth]{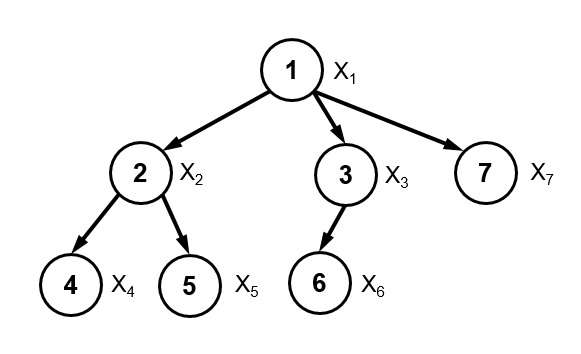}
\caption{Example of a tree-structured network. Each node has one
parent, except for $X_1$, the root.}
\label{fig:bn-tree}
\end{figure}

\subsection{Tree Structured Network}
\label{sec:tree}
When the Bayesian network is structured as a tree, each node has
exactly one parent, except for the single root\footnote{We assume
  that the graph is connected, but this can be easily generalized for
  the case of a forest.}. An example of tree-structured network is
shown in Figure~\ref{fig:bn-tree}.
The following result is a consequence of Theorem~\ref{thm:nonuniform}
specialized to a tree, by noting that each set
$\parent{X_i}$ is of size $1$, 
we let $J_{\parent{i}}$ denote $K_i$, the cardinality of $\parent{X_i}$.
\begin{lemma}
\label{lem:bn-tree}
Given $0<\eps,\delta<1$ and a tree-structured network with $n$ variables, 
Algorithm \nonuniform can continuously maintain an $(\eps,\delta)$-approximation to the MLE 
incurring communication cost 
$O( \Gamma \cdot \frac{\sqrt{k}}{\eps} \cdot \log \frac{1}{\delta} \cdot \log m  )$ messages.
where 
$\Gamma = ( \sum_{i=1}^n (J_i J_{\parent{i}})^{2/3} )^{3/2}   +  ( \sum_{i=1}^n J_{\parent{i}}^{2/3} )^{3/2}$.
For the case when $J_i=J$ for all $i$, this reduces to $\Gamma=O(n^{1.5} J^2)$.
\end{lemma}

\subsection{Na{\"i}ve Bayes}
\label{sec:naive}
The \naivebayes model is perhaps the most commonly used graphical
model, especially in tasks such as classification, and has a simple
structure as shown in Figure~\ref{fig:bn-naive}. The graphical model
of \naivebayes is a two-layer tree where we assume the root is node $1$.

\begin{figure}
\centering
\includegraphics[width=0.5\textwidth]{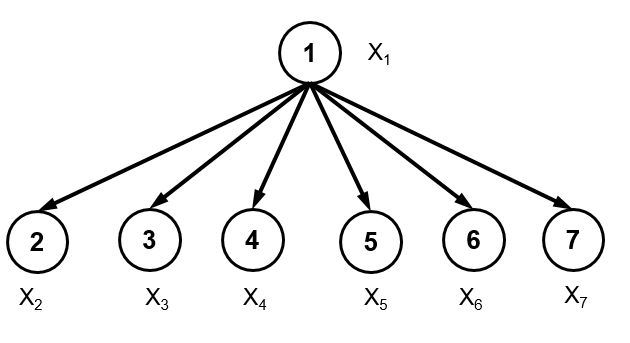}
\caption{Example of \naivebayes over variables $X_1, X_2 ..., X_7$.
All the nodes except $X_1$ in the graph have the same parent $X_1$, which
is the root of the tree.}
\label{fig:bn-naive}
\end{figure}

Specializing the \nonuniform algorithm for the case of
\naivebayes, we use results \eqref{eq:non-uniform-sol1}
and~\eqref{eq:non-uniform-sol2}.
For each node $X_i$ with $i \in [2,n]$, $K_i=J_1$.
Hence, we have the approximation factors $\epsfnA(i) = \nu_i$ and
$\epsfnB(i) = \mu_i$ as follows.
%
\begin{equation}
\label{eq:naive-sol}
\nu_i = {\textstyle{\frac{\eps}{16}}}J_i^{1/3} \Big/\big( \sum_{i=2}^n J_i^{2/3} \big)^{1/2}, \qquad
\mu_i = \frac{\eps}{16\sqrt{n}}
\end{equation}


Note that we maintain the counter $\dc_i(x_1)$ 
for each $x_1 \in \domain(X_1)$ and $i \in [2,n]$ independently, 
as $X_1$ is the parent of $X_2, X_3, \ldots, X_n$.
This is wasteful since for $i \in [2, n]$, 
$\dc_i(x_1)$ are all tracking the same event.
Utilizing this special structure, we can do better by maintain only
one copy of the counter $\dc(x_1)$ for each $x_1 \in \domain(X_1)$,
but with a more accurate approximation factor $\frac{\eps}{3n}$. The
resulting algorithm uses Algorithm~\ref{algo:naive-bayes} to perform
initialization. 


\begin{algorithm}[t]
\label{algo:naive-bayes}
\caption{Naive-Bayes-Init()} 
\ForEach{$i=2 \ldots n$, $x_i \in \domain(X_i)$, and $x_1 \in \domain(X_1)$}
{$\dc_i(x_i, x_1) \gets$ DistCounter($\nu_i, \delta$), where $\nu_i$ is shown in Equation~\ref{eq:naive-sol}}
\ForEach{$x_1 \in \domain(X_1)$}
{$\dc_1(x_1) \gets$ DistCounter($\frac{\eps}{3n}, \delta$)}
\end{algorithm}

\begin{lemma}
\label{lem:naive-bayes}
Given $0<\eps,\delta<1$ and a \naivebayes model with $n$ variables,
Algorithm~\ref{algo:naive-bayes} combined with Algorithms~\ref{algo:generic-train} and \ref{algo:generic-query}
continuously maintains an $(\eps,\delta)$-approximation to the MLE,
incurring communication cost over all the observations
$O\left(\frac{\sqrt{k}}{\eps} \cdot J_1 \cdot \left( \sum_{i=2}^n J_i^{2/3} \right)^{3/2} \cdot \log \frac{1}{\delta} \cdot \log m \right)$
messages over $m$ distributed observations. In the case when all $J_i$ are equal to $J$, this expression is 
$O\left(\frac{n^{3/2} \sqrt{k}}{\eps} \cdot J^2 \cdot \log \frac{1}{\delta} \cdot \log m \right)$.
\end{lemma}

\subsection{Classification}
\label{sec:class}
Thus far, our goal has been to estimate probabilities of joint distributions of random variables.
We now present an application of these techniques to the task of Bayesian classification. 
In classification, we are given some evidence $e$, 
and the objective is to find an assignment to a subset of random variables $Y$, given $e$.
The usual way to do this is to find the assignment that maximizes the
probability, given $e$. That is,
$\text{Class}(\bm{Y} \mid e) = \argmax_y \prob{y,e}$. 
We are interested in an approximate version of the above formulation,
given by:

\begin{definition}
\label{defn:class-obj}
Given a \bayesnet $\mathcal{G}$, let $Y$ denote the set of variables
whose values need to be assigned, and $\eps$ denote an error parameter. For any
evidence $e$, we say that $\bm{b}$ solves Bayesian classification with
$\eps$ error if
\[\textstyle\empProb{Y= \bm{b} \mid e} \ge (1-\eps) \cdot \max_y
\empProb{Y=y \mid e}.\]
\end{definition}
In other words, we want to find the assignment to the set of variables
$Y$ with conditional probability close to the maximum, if not equal to
the maximum.

\begin{lemma}
\label{lem:subset-vars}
If for a set of variables $\mathcal{X} = \{X_1,\ldots,X_n\}$, we have 
$e^{-\eps/2} \leq \frac{\estProb{\mathcal{X}}}{\empProb{\mathcal{X}}} \leq e^{\eps/2}$,
then for any subset of non-overlapping variables $\bm{X},\bm{Y} \subseteq \mathcal{X}$, $\bm{X} \cap \bm{Y} = \emptyset$,
$e^{-\eps} \leq \frac{\estProb{\bm{Y} \mid \bm{X}}}{\empProb{\bm{Y} \mid \bm{X}}} \leq e^{\eps}$.
\end{lemma}
\begin{proof}
For variable set $\bm{X} \subseteq \mathcal{X}$, we have
\[
e^{-\eps/2} \leq \frac{\estProb{\bm{X}}}{\empProb{\bm{X}}} \leq e^{\eps/2}
\]
Similarly, for variable set $\{ \bm{X}, \bm{Y} \} \subseteq \mathcal{X}$, we have
\[
e^{-\eps/2} \leq \frac{\estProb{\bm{X}, \bm{Y}}}{\empProb{\bm{X}, \bm{Y}}} \leq e^{\eps/2}
\]
Combining above two inequations,
\[
e^{-\eps} \leq \frac{\estProb{\bm{X}, \bm{Y}}}{\estProb{\bm{X}}} \cdot 
\frac{\empProb{\bm{X}}}{\empProb{\bm{X}, \bm{Y}}} \leq e^{\eps}
\]
Applying  Bayes rule, we complete our proof.
\end{proof}

\begin{lemma}
\label{lem:classification}
Given evidence $e$ and set of variables $Y$,
if $e^{-\eps / 4} \leq
\frac{\estProb{\mathcal{X}}}{\empProb{\mathcal{X}}} \leq e^{\eps / 4}$, then
we can find assignment $\bm{b}$ that solves the Bayesian classification
problem with $\eps$ error.
\end{lemma}
\begin{proof}
Let $\bm{b} = \argmax_y \estProb{Y=y \mid e}$ and $\bm{b}^* = \argmax_y \empProb{Y=y \mid e}$.
From Lemma~\ref{lem:subset-vars}, we have
\[
e^{\eps/2} \cdot \empProb{\bm{Y}=\bm{b} \mid e} \geq \estProb{\bm{Y}=\bm{b} \mid e}
\]
As $\bm{b}$ is the most likely assignment for $\estProb{Y=y \mid e}$,
\[
\estProb{\bm{Y}=\bm{b} \mid e} \geq \estProb{\bm{Y}=\bm{b}^* \mid e}
\]
From Lemma~\ref{lem:subset-vars}, for assignment $\bm{b}^*$, we have
\[
\estProb{\bm{Y}=\bm{b}^* \mid e} \geq e^{-\eps/2} \cdot \empProb{\bm{Y}=\bm{b}^* \mid e}
\]
So we can derive that
\[
\empProb{\bm{Y}=\bm{b} \mid e} \geq e^{-\eps} \cdot \empProb{\bm{Y}=\bm{b}^* \mid e}
\]
\end{proof}

\begin{theorem}
There is an algorithm for Bayesian
classification (Definition~\ref{defn:class-obj}), with communication
$O\left( \Gamma \cdot \frac{\sqrt{k}}{\eps} \cdot \log \frac{1}{\delta} \cdot \log m \right)$ messages 
over $m$ distributed observations, where $\Gamma = \left( \sum_{i=1}^n (J_i K_i)^{2/3}
\right)^{3/2} + \left( \sum_{i=1}^n K_i^{2/3} \right)^{3/2}$.
\end{theorem}

\begin{proof}
We use \nonuniform to maintain distributed counters with error
factor $\frac{\eps}{4}$. From Theorem~\ref{thm:nonuniform}, we have $e^{-\eps / 4} \leq
\frac{\estProb{\mathcal{X}}}{\empProb{\mathcal{X}}} \leq e^{\eps / 4}$
where $\mathcal{X}$ denote all the variables. Then from
Lemma~\ref{lem:classification}, we achieve our goal of Bayesian
classification with $\eps$ error.
\end{proof}

\section{Experimental Evaluation}
\label{sec:experiment}
\newcommand{\alarm}{\texttt{ALARM}\xspace}
\newcommand{\newalarm}{\texttt{NEW-ALARM}\xspace}
\newcommand{\hepar}{\texttt{HEPAR II}\xspace}
\newcommand{\link}{\texttt{LINK}\xspace}
\newcommand{\munin}{\texttt{MUNIN}\xspace}

\subsection{Setup and Implementation Details}
Algorithms were implemented in Java with JDK version 1.8,
and evaluated on a 64-bit Ubuntu Linux machine with
Intel Core i5-4460 3.2GHz processor and 8GB RAM. 

\tpara{Datasets:} We use real-world Bayesian networks from
the repository of Bayesian networks at~\cite{bnlearn}. In our
experiments, algorithms assume the network topology, but learn model
parameters from training data. Based on the number of nodes in the
graph, networks in the dataset are classified into five categories:
small networks ($<20$ nodes), medium networks ($20-60$ nodes), large
networks ($60-100$ nodes), very large networks ($100-1000$ nodes) and
massive networks ($>1000$ nodes). We select one medium network
\alarm~\cite{alarm}, one large network \hepar~\cite{hepar}, one very
large network \link~\cite{link} and one massive network
\munin~\cite{munin}. Table~\ref{table:dataset} provides an overview of
the networks that we use.

\begin{table}
\centering
\caption{Bayesian Networks used in the experiments.}
{\footnotesize
\centering
\begin{tabular}
{| l | r | r | r |}
\toprule
Dataset & Number    & Number    & Number of  \\ 
              & of Nodes  & of Edges   & Parameters \\
\midrule
\alarm~\cite{alarm} & $37$ & $46$ & $509$ \\
\hepar~\cite{hepar} & $70$ & $123$ & $1453$ \\
\link~\cite{link} & $724$ & $1125$ & $14211$ \\
\munin~\cite{munin} & $1041$ & $1397$ & $80592$ \\
\bottomrule
\end{tabular}
\label{table:dataset}}
\end{table}

\begin{figure*}[t]
  \centering
  \subfigure[Exact]{
  \includegraphics[width=0.25\textwidth]{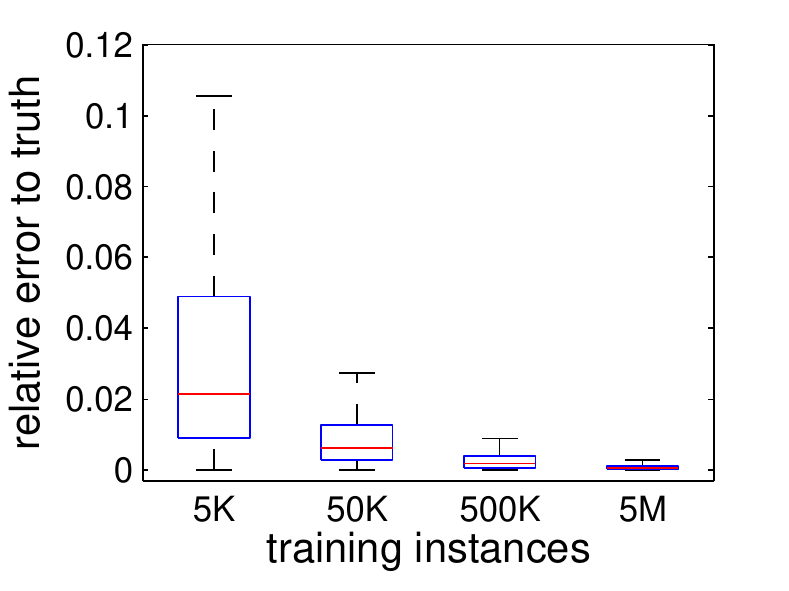}
  }%
  \subfigure[Baseline Approx.]{
  \includegraphics[width=0.25\textwidth]{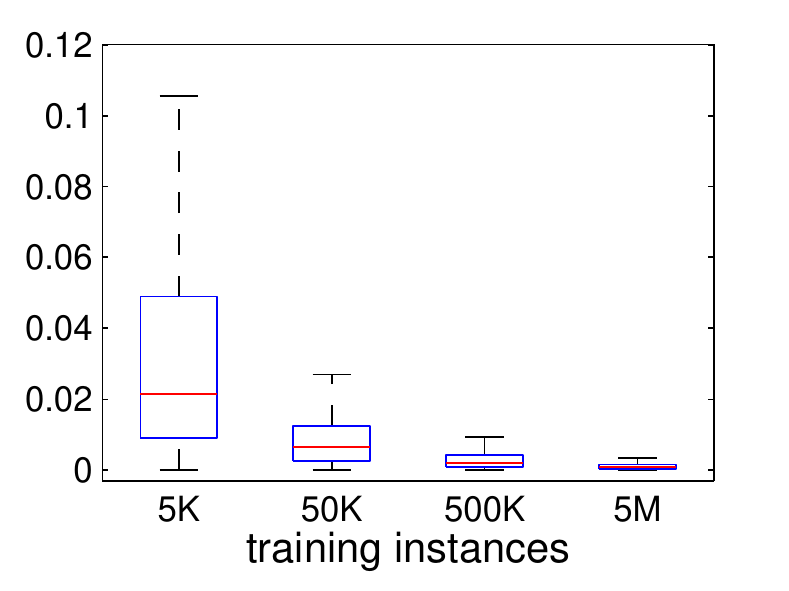}
  }%
  \subfigure[Uniform Approx.]{
  \includegraphics[width=0.25\textwidth]{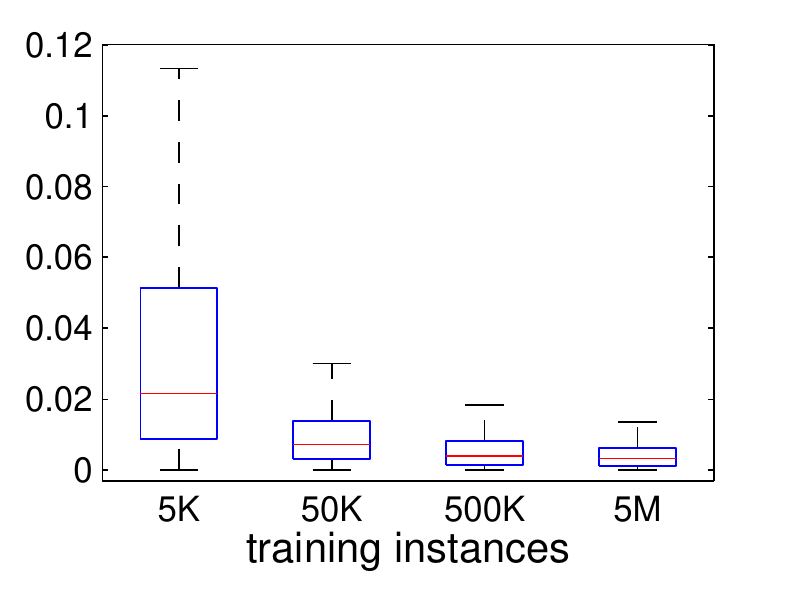}
  }%
  \subfigure[Non-uniform Approx.]{
  \includegraphics[width=0.25\textwidth]{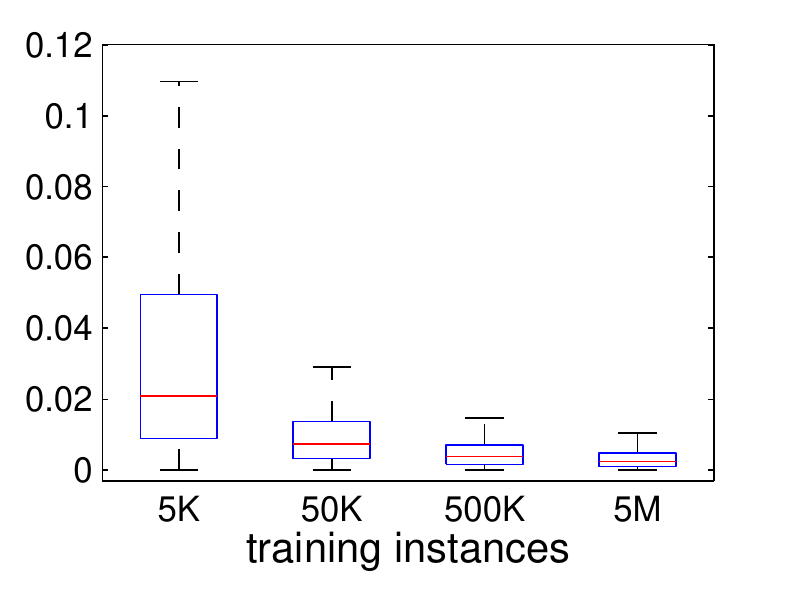}
  }
\caption{Testing error (relative to the ground truth) vs. number of training points. The dataset is \hepar.}
\label{fig:error-vs-stream-hepar}
\end{figure*}
\begin{figure*}[t]
  \centering
  \subfigure[Exact]{
  \centering \includegraphics[width=0.25\textwidth]{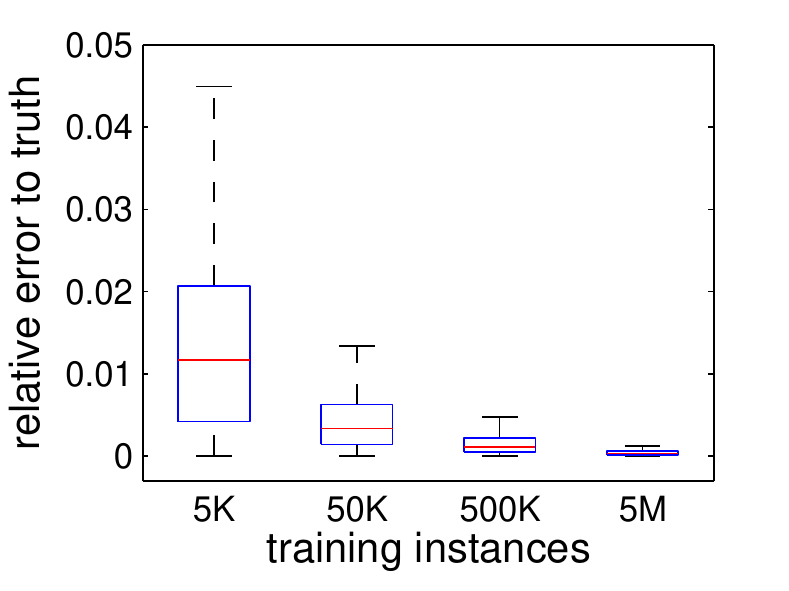}
  }%
  \subfigure[Baseline Approx.]{
  \centering \includegraphics[width=0.25\textwidth]{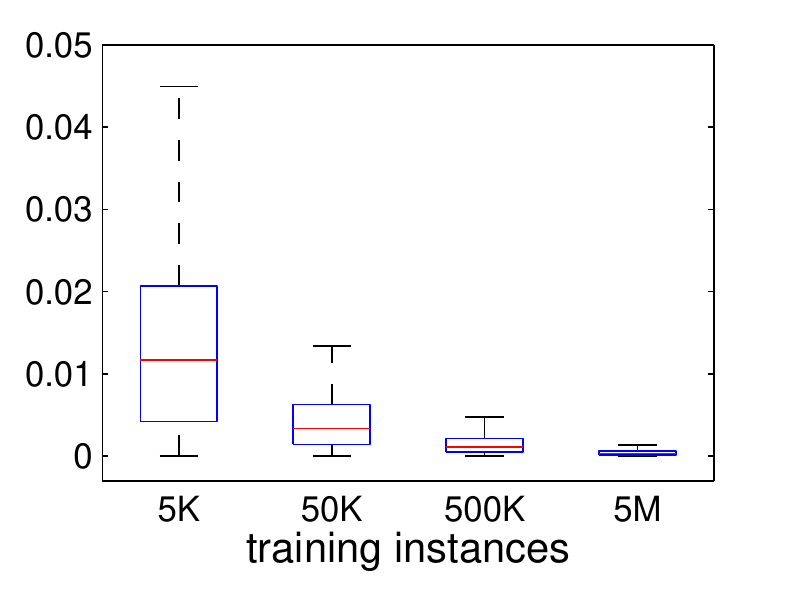}
  }%
  \subfigure[Uniform Approx.]{
  \centering \includegraphics[width=0.25\textwidth]{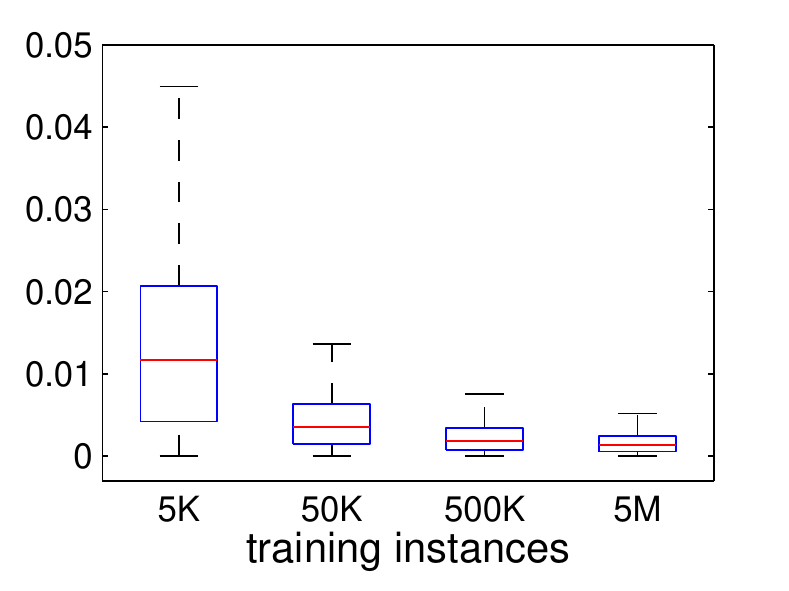}
  }%
  \subfigure[Non-uniform Approx.]{
  \centering \includegraphics[width=0.25\textwidth]{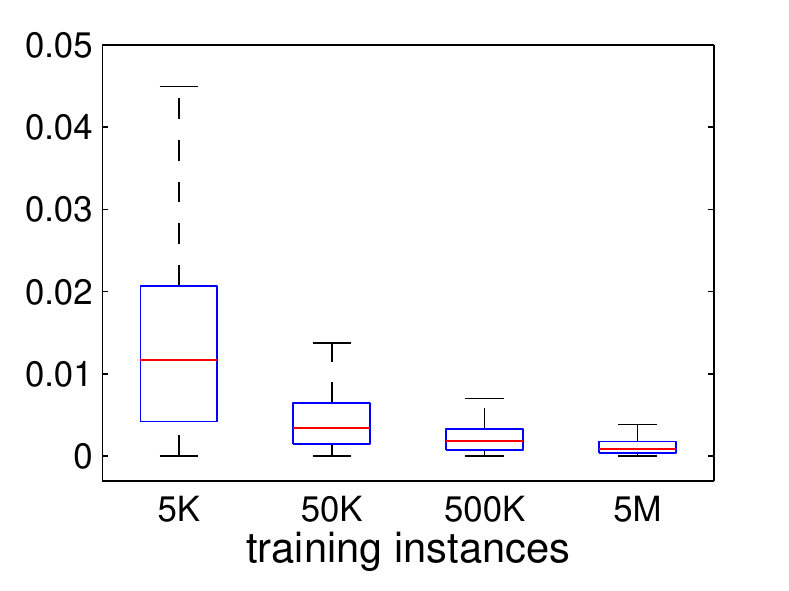}
  }
\caption{Testing error (relative to the ground truth) vs. number of training points. The dataset is \link.}
\label{fig:error-vs-stream-link}
\end{figure*}

\begin{figure*}[t]
  \subfigure[\alarm]{
  \includegraphics[width=0.25\textwidth]{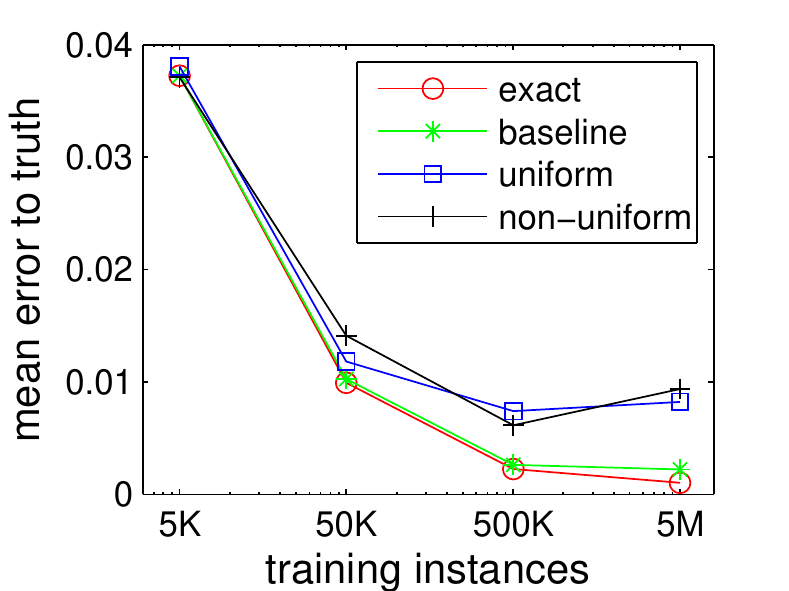}
  }%
  \subfigure[\hepar]{
   \includegraphics[width=0.25\textwidth]{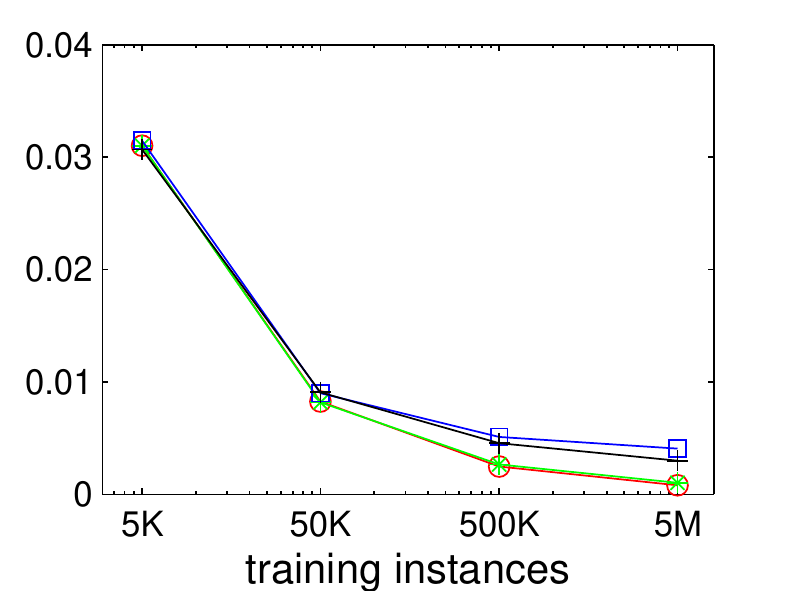}
  }%
  \subfigure[\link]{
   \includegraphics[width=0.25\textwidth]{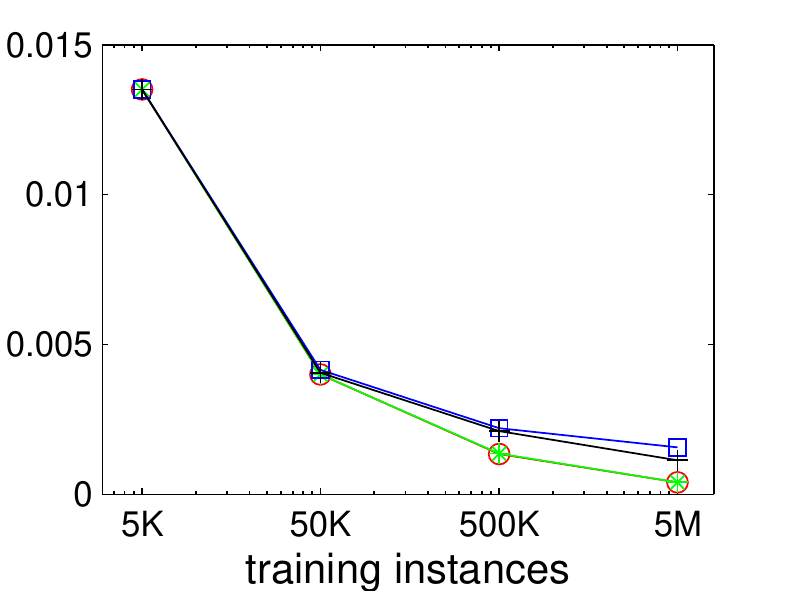}
  }%
  \subfigure[\munin]{
   \includegraphics[width=0.25\textwidth]{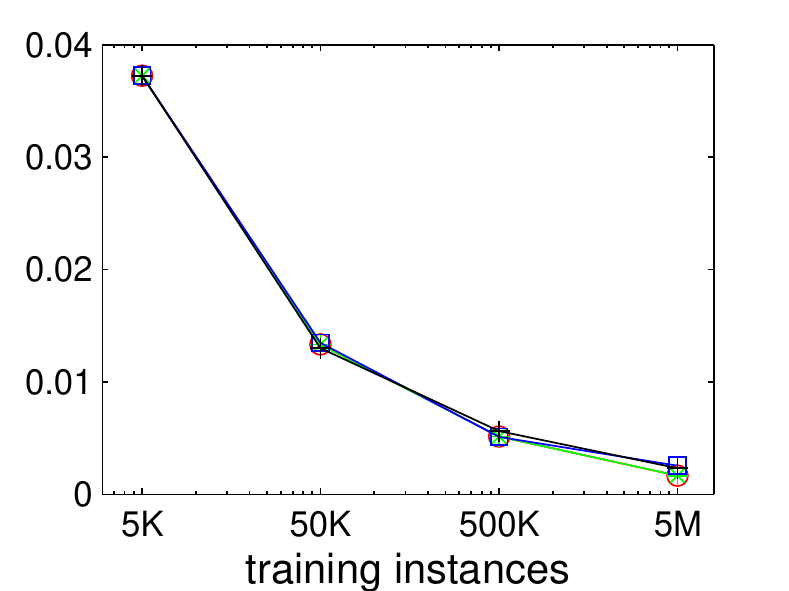}
  }
\caption{Mean testing error (relative to the ground truth) vs. number of training points.}
\label{fig:avgerr-vs-stream}
\end{figure*}

\tpara{Training Data:} For each network, we generate training data
based on the ground truth for the parameters. To do this, we first
generate a topological ordering of all vertices in the Bayesian
network (which is guaranteed to be acyclic), and then assign values to
nodes (random variables) in this order, based on the known conditional probability distributions. 

\tpara{Testing Data:} Our testing data consist of a number of queries, each one for the probability of a specific event. We measure the accuracy according to the ability of the network to accurately estimate the probabilities of different events. To do this, we generate $1000$ events on the joint probability space represented by the Bayesian network, and estimate the probability of each event using the parameters that have been learnt by the distributed algorithm. Each event is chosen so that its ground truth probability is at least $0.01$ -- this is to rule out events that are highly unlikely, for which not enough data may be available to estimate the probabilities accurately.

\tpara{Distributed Streams:} We built a simulator for distributed
stream monitoring, which simulates a system of $k$ sites and a single
coordinator.
All events (training data) arrive at sites, and queries
are posed at the coordinator.
Each data point is sent to a site chosen uniformly at random.

\tpara{Algorithms:} We implemented four algorithms: \exact, 
\baseline, \uniform, and \nonuniform. \exact is the strawman algorithm
that uses exact counters so that each site informs the coordinator
whenever it receives a new observation. 
This algorithm sends a message for each
counter, so that the length of each message exchanged is approximately
the same. This makes the measurement of communication cost across
different algorithms equivalent to measuring the number of
messages. The other three algorithms, \baseline, \uniform, and
\nonuniform, are as described in Sections~\ref{sec:baseline},
\ref{sec:uniform}, and~\ref{sec:nonuniform} respectively. For each of
these algorithms, a message contains an update to the value of a
single counter.

\tpara{Metrics: }
We compute the probability for each testing event
using the approximate model maintained by the distributed
algorithm.
We compare this with the ground truth probability for the
testing event, derived from the ground truth model. For \baseline,
\uniform, and \nonuniform, we compare their results with those
obtained by \exact, and report the median value from five independent
runs. 
Unless otherwise specified, we set $\epsilon = 0.1$ and the number of
sites to $k=30$.

\begin{figure*}[t]
  \centering
  \subfigure[Baseline Approx.]{
  \centering \includegraphics[width=0.25\textwidth]{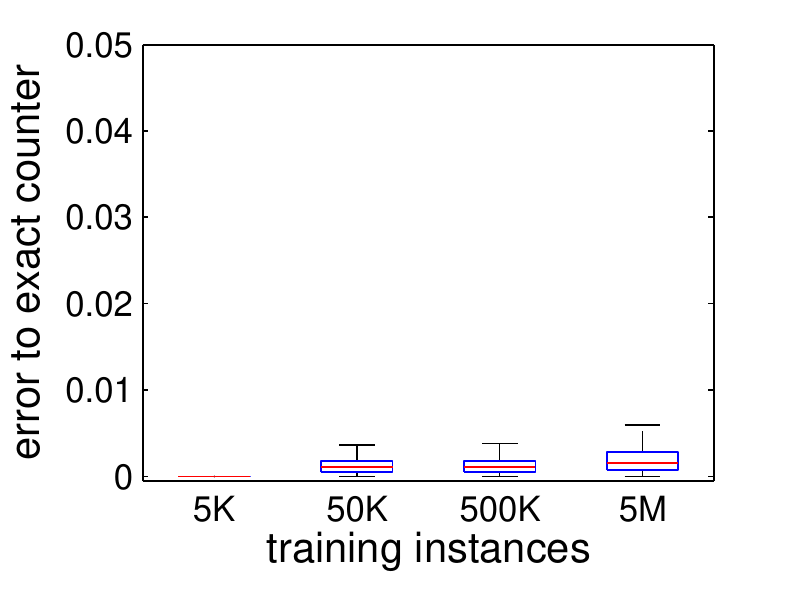}
  }%
  \subfigure[Uniform Approx.]{
  \centering \includegraphics[width=0.25\textwidth]{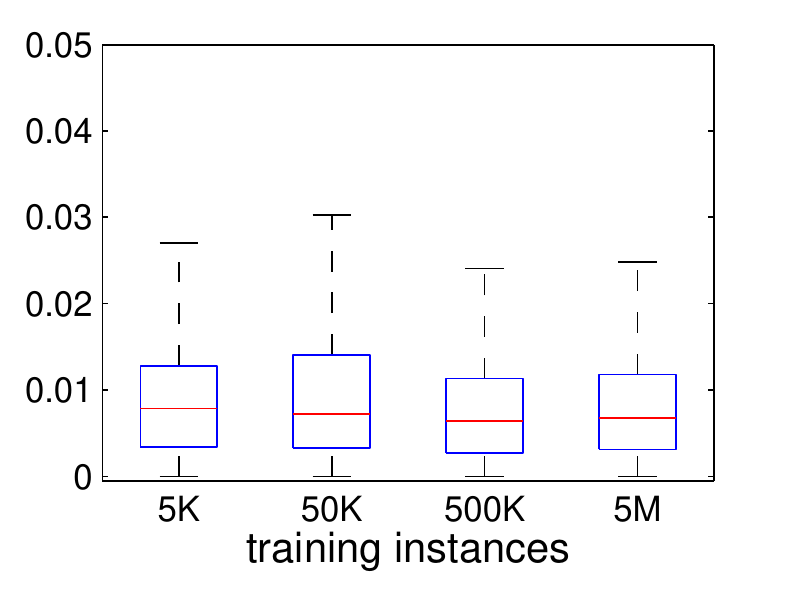}
  }%
  \subfigure[Non-uniform Approx.]{
  \centering \includegraphics[width=0.25\textwidth]{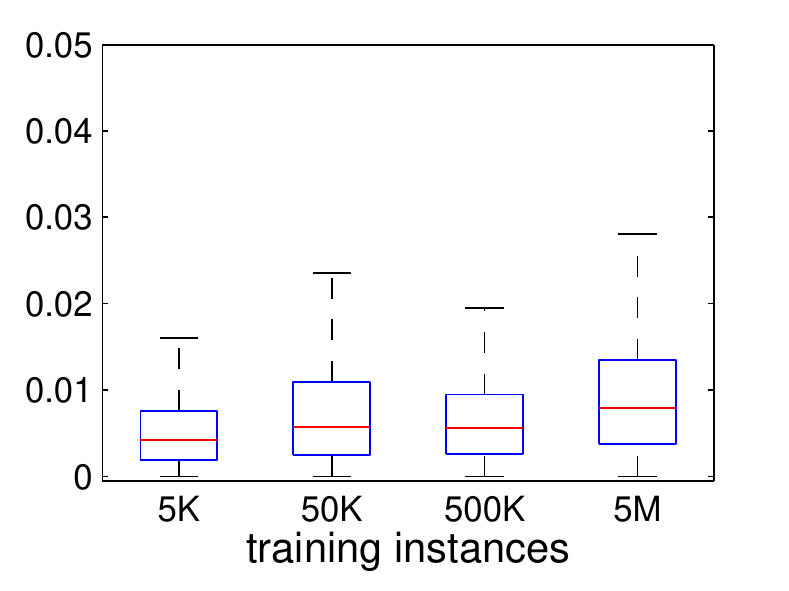}
  }
\caption{Testing error (relative to \exact) vs. number of training points. The dataset is \alarm.}
\label{fig:error-to-exact-alarm}
\end{figure*}
\begin{figure*}[t]
  \centering
  \subfigure[Baseline Approx.]{
  \centering \includegraphics[width=0.25\textwidth]{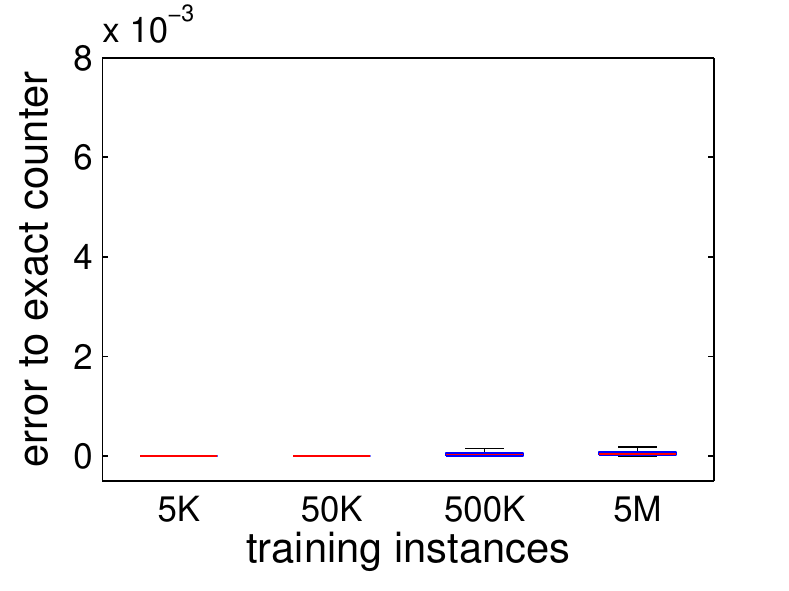}
  }%
  \subfigure[Uniform Approx.]{
  \centering \includegraphics[width=0.25\textwidth]{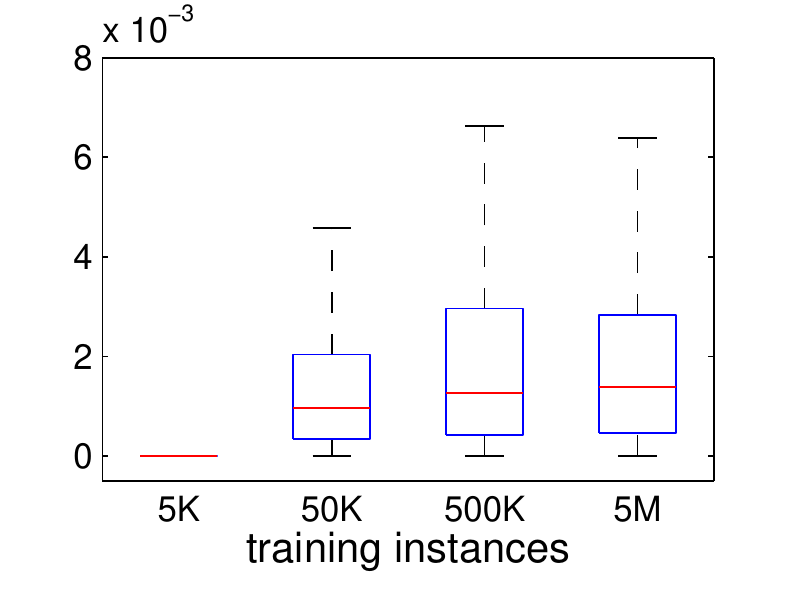}
  }%
  \subfigure[Non-uniform Approx.]{
  \centering \includegraphics[width=0.25\textwidth]{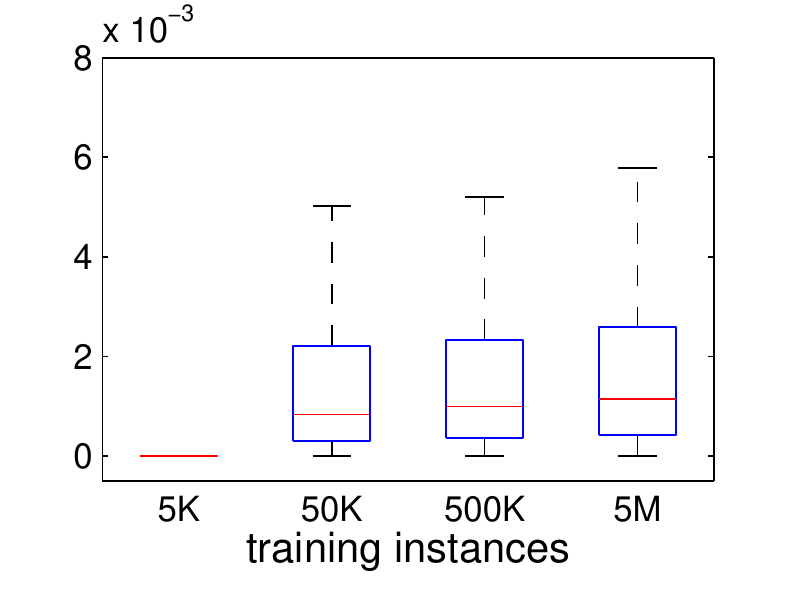}
  }%
\caption{Testing error (relative to \exact) vs. number of training points. The dataset is \munin.}
\label{fig:error-to-exact-munin}
\end{figure*}

\begin{figure*}[t]
  \subfigure[\alarm]{
   \includegraphics[width=0.25\textwidth]{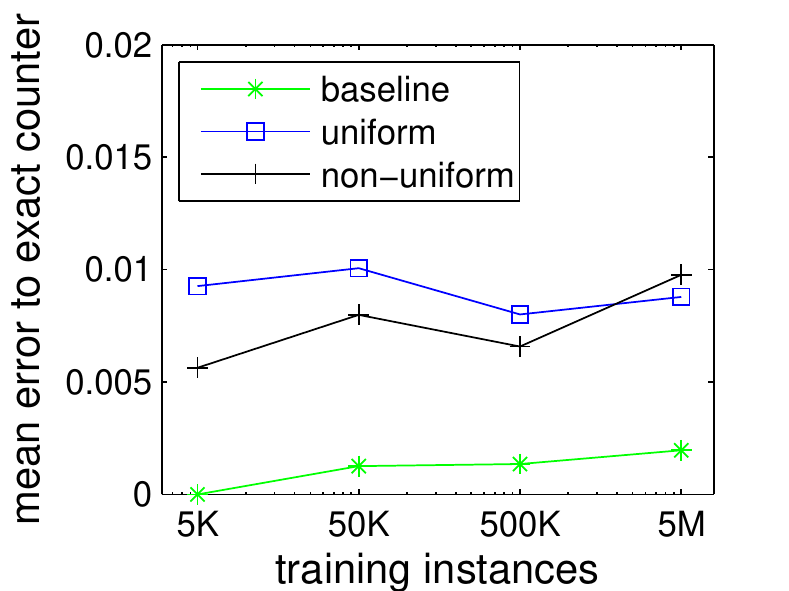}
  }%
  \subfigure[\hepar]{
   \includegraphics[width=0.25\textwidth]{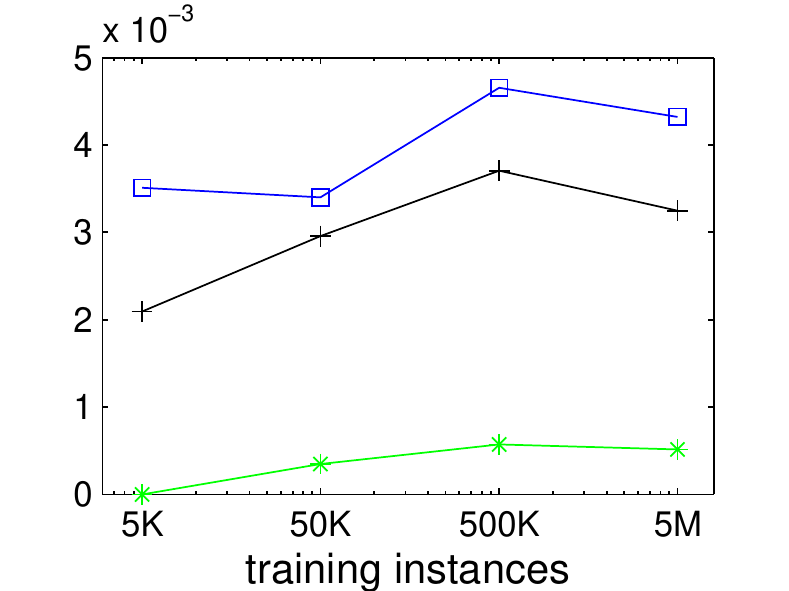}
  }%
  \subfigure[\link]{
   \includegraphics[width=0.25\textwidth]{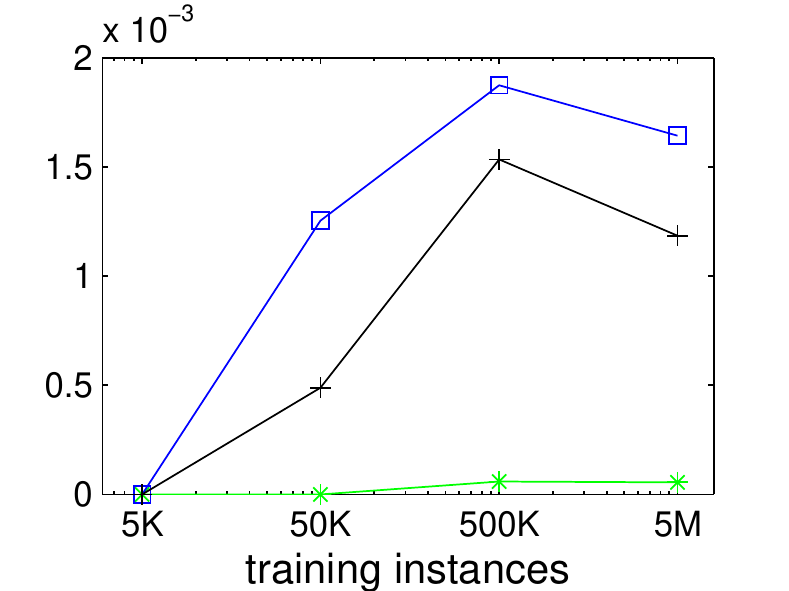}
  }%
  \subfigure[\munin]{
   \includegraphics[width=0.25\textwidth]{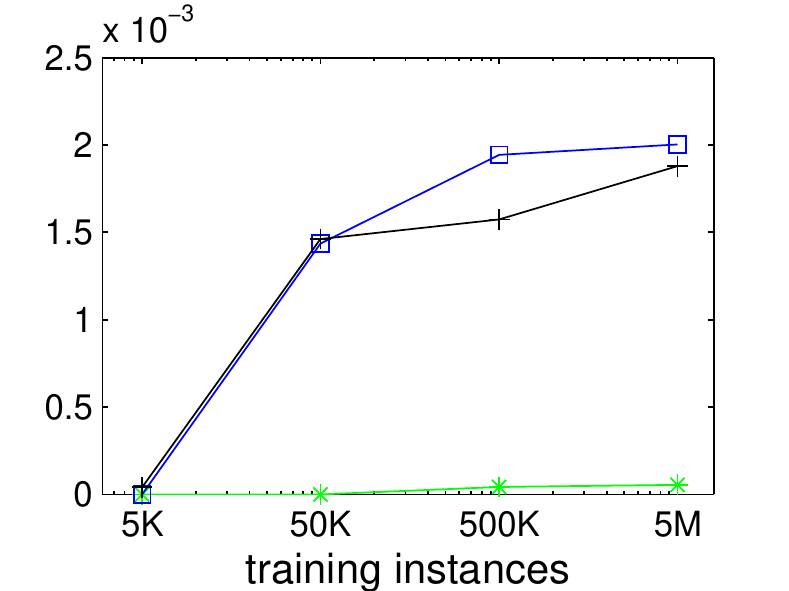}
  }
\caption{Mean testing error (relative to \exact) vs. number of training points.}
\label{fig:avgerr-to-exact}
\end{figure*}




\subsection{Results and Discussion}


\tpara{The error relative to the ground truth} is the average error
of the probability estimate returned by the model learnt by the
algorithm, relative to the ground truth probability, computed using
our knowledge of the underlying
model. Figures~\ref{fig:error-vs-stream-hepar}
and~\ref{fig:error-vs-stream-link} respectively show this error as a
function of the number of training instances, for the \hepar and \link
datasets respectively. As expected, for each algorithm, the median
error decreases with an increase in the number of training instances,
as can be seen by the middle quantile in the boxplot. The
interquartile ranges also shrink with more training instances, showing
that the variance of the error is also decreasing.



Figure~\ref{fig:avgerr-vs-stream}
shows the relative performance of the algorithms.
\exact has the best performance of all
algorithms, which is to be expected, since it computes the model
parameters based on exact counters. \baseline has the next best
performance, closely followed by \uniform and \nonuniform, which have
quite similar performance.
Note that the slightly better
performance of \exact and \baseline comes at a vastly increased
message cost, as we will see next.
Finally, all these algorithms achieve good accuracy results.
For instance, after $5M$ examples, the error in estimated event
probabilities is always less than one percent, for any of these
algorithms.

\tpara{The error relative to the MLE} is the average error of
the probability estimate returned by the model learnt by the
algorithm, relative to the model learnt using  exact counters.
The distribution of this error is shown for the different
algorithms that use approximate counters, in
Figures~\ref{fig:error-to-exact-alarm}
and~\ref{fig:error-to-exact-munin} for the \alarm and \munin datasets
respectively. The mean error for different algorithms is plotted in
Figure~\ref{fig:avgerr-to-exact}.  We can consider the measured error
 as having two sources:
(1)~Statistical error, which is the error in learning that is inherent
due to the number of training examples seen so far -- this is captured
by the error of the model learnt by the exact counter, relative to the
ground truth, and
(2)~Approximation error, which is the difference
between the model that we are tracking and the model learnt by using
exact counters -- this error arises due to our desire for efficiency
of communication (i.e., trying to send fewer messages for
counter maintenance).
Our algorithms aim to control the
approximation error, and this error is captured by the
error relative to exact counter.
We note from the plots that the error
relative to exact counter remains approximately the same with
increasing number of training points, for all three algorithms,
\baseline, \uniform, and \nonuniform.
This is consistent with
theoretical predictions since our algorithms only guarantee that these
errors are less than a threshold ($\epsilon$), which does not decrease
with increasing number of points.
The error of \nonuniform is
marginally better than that of \uniform.
We emphasise that
error relative to the ground truth is a more important metric than the
error relative to MLE.

\begin{figure*}[t]
  \centering
  \subfigure[\alarm]{
  \centering \includegraphics[width=0.25\textwidth]{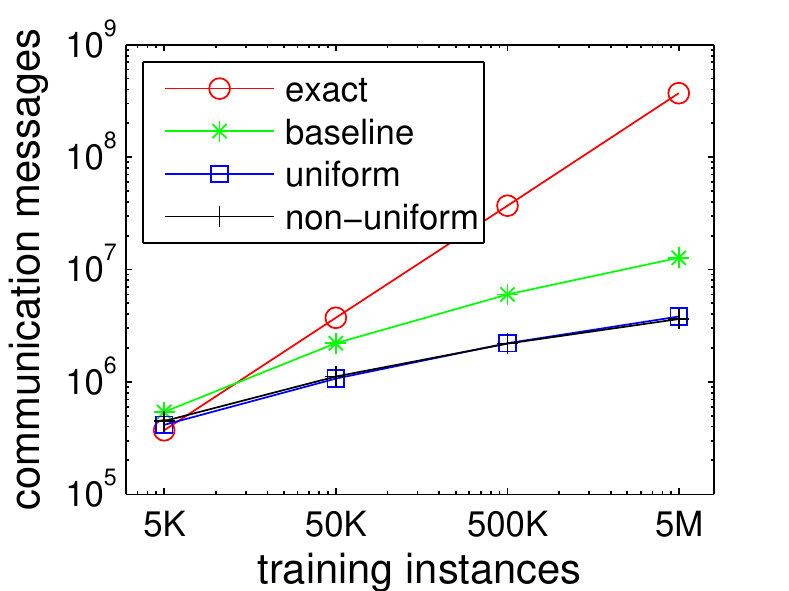}
  }%
  \subfigure[\hepar]{
  \centering \includegraphics[width=0.25\textwidth]{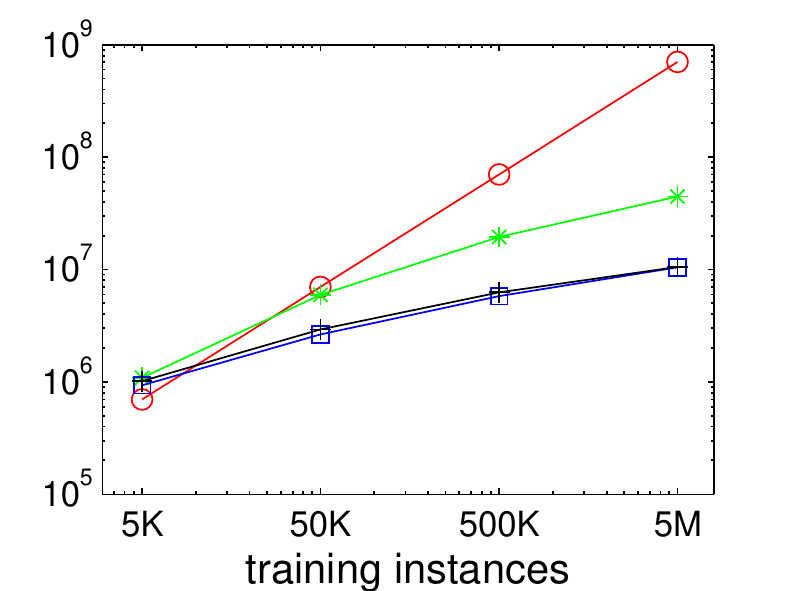}
  }%
  \subfigure[\link]{
  \centering \includegraphics[width=0.25\textwidth]{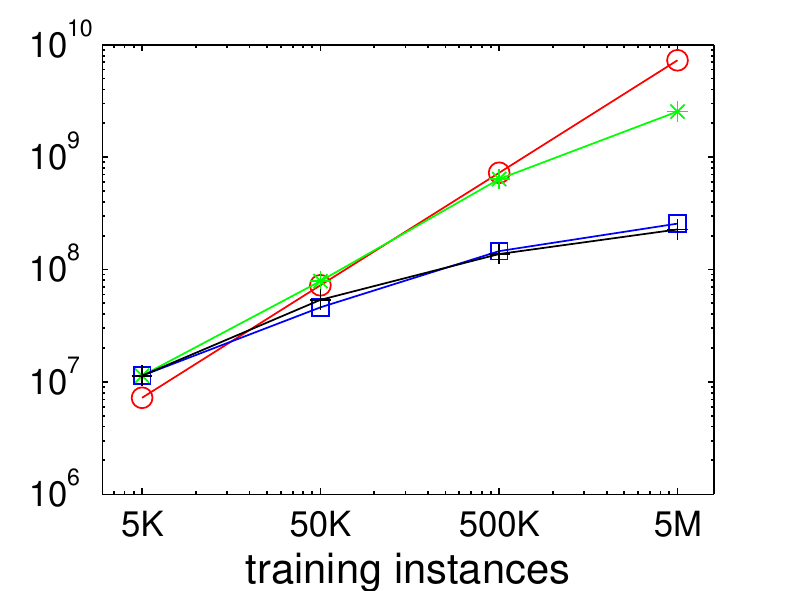}
  }%
  \subfigure[\munin]{
  \centering \includegraphics[width=0.25\textwidth]{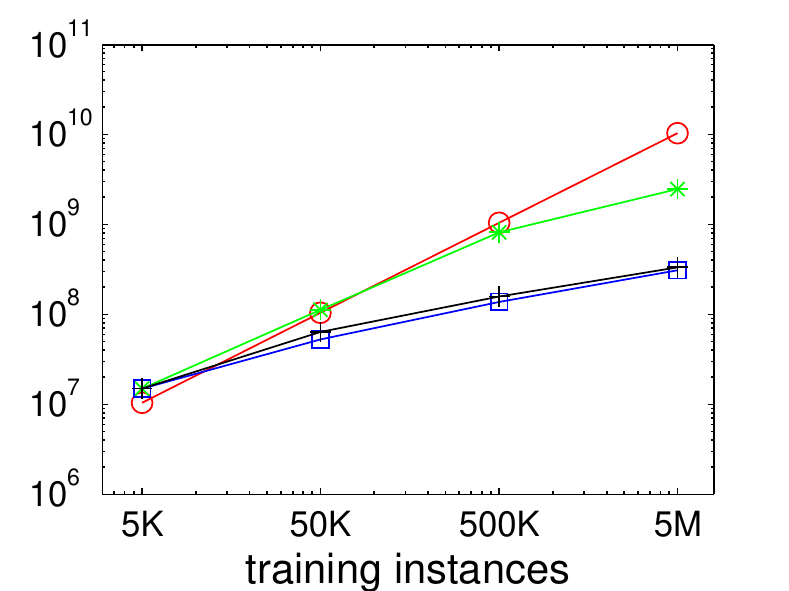}
  }
\caption{Communication cost vs. number of training points.}
\label{fig:comm-vs-stream}
\end{figure*}

\tpara{Communication cost versus the number of training points} for
different algorithms is shown in Figure~\ref{fig:comm-vs-stream}. Note
that the y-axis is in logarithmic scale. From this graph, we can
observe that \nonuniform has the smallest communication cost in
general, followed by \uniform. These two have a significantly smaller
cost than \baseline and \exact. The gap between \exact and \nonuniform
increases as more training data arrives. For 5M training points,
\nonuniform sends approximately 100 times fewer messages than \exact,
while having almost the same accuracy when compared with the ground
truth. This shows the benefit of using approximate counters in
maintaining the Bayesian network model. It also shows that there is a
concrete and tangible benefit using the improved analysis in
\uniform and \nonuniform, in reducing the communication cost.

\begin{figure*}[t]
  \centering
  \subfigure[Baseline Approx.]{
  \centering \includegraphics[width=0.25\textwidth]{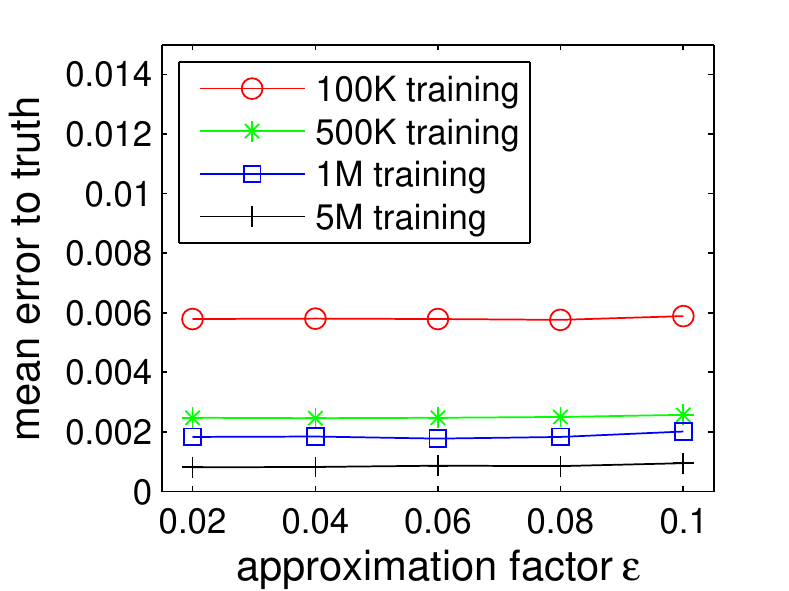}
  }%
  \subfigure[Uniform Approx.]{
  \centering \includegraphics[width=0.25\textwidth]{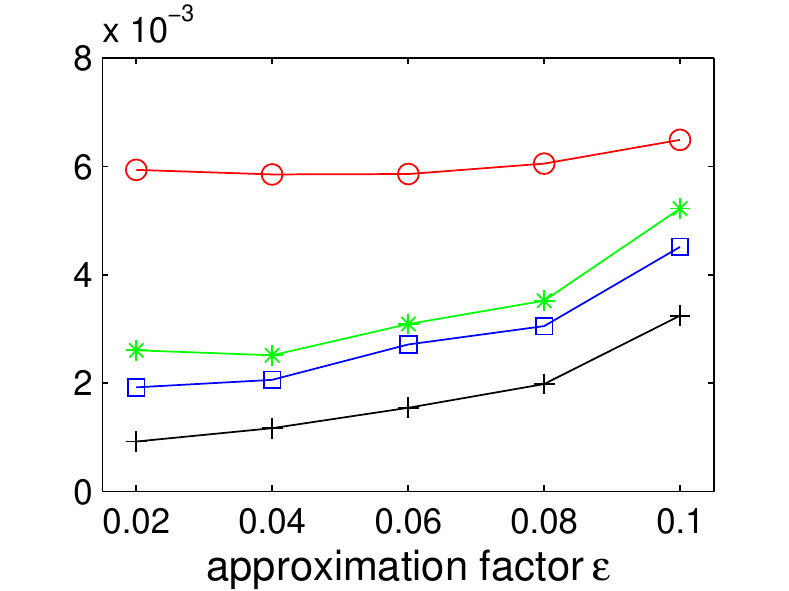}
  }%
  \subfigure[Non-uniform Approx.]{
  \centering \includegraphics[width=0.25\textwidth]{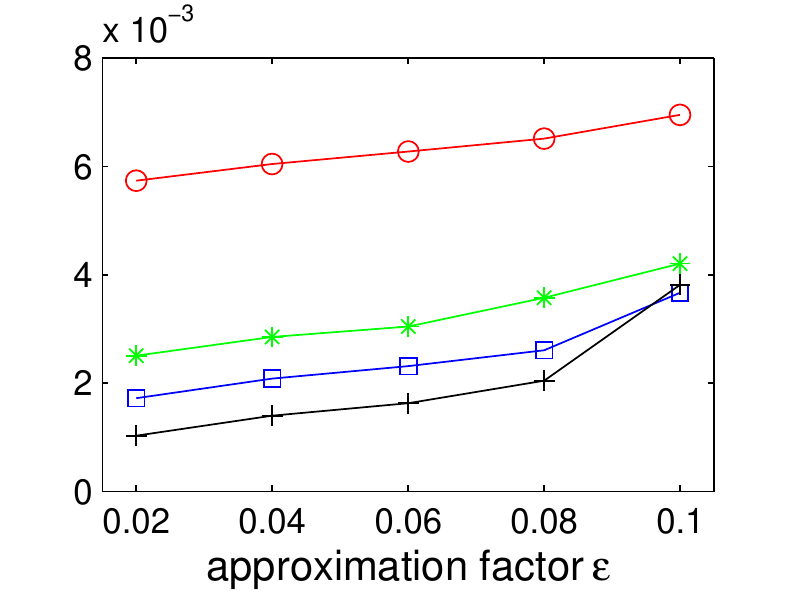}
  }
\caption{Mean testing error (relative to ground truth) vs. approximation factor $\eps$. The dataset is \hepar.}
\label{fig:error-vs-eps}
\end{figure*}

Figure~\ref{fig:error-vs-eps} shows the testing error as a function of
the parameter $\eps$, and shows that the testing error increases with
an increase in $\eps$. In some cases, the testing error does not
change appreciably as $\eps$ increases. This is due to the fact that
$\eps$ only controls the ``approximation error", and in cases when the
statistical error is large (i.e. small numbers of training instances),
the approximation error is dwarfed by the statistical error, and the
overall error is not sensitive to changes in $\eps$.

\begin{figure}[t]
  \centering
  \subfigure[Communication cost for \alarm]{
 \includegraphics[width=0.25\textwidth]{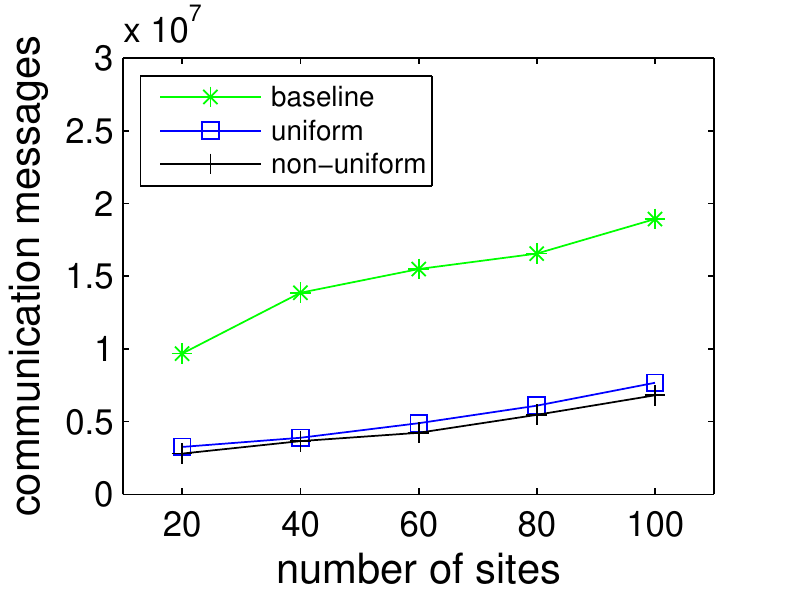}
 \label{fig:msg-vs-sites}
  }%
\subfigure[Communication cost for \newalarm]{
\includegraphics[width=0.25\textwidth]{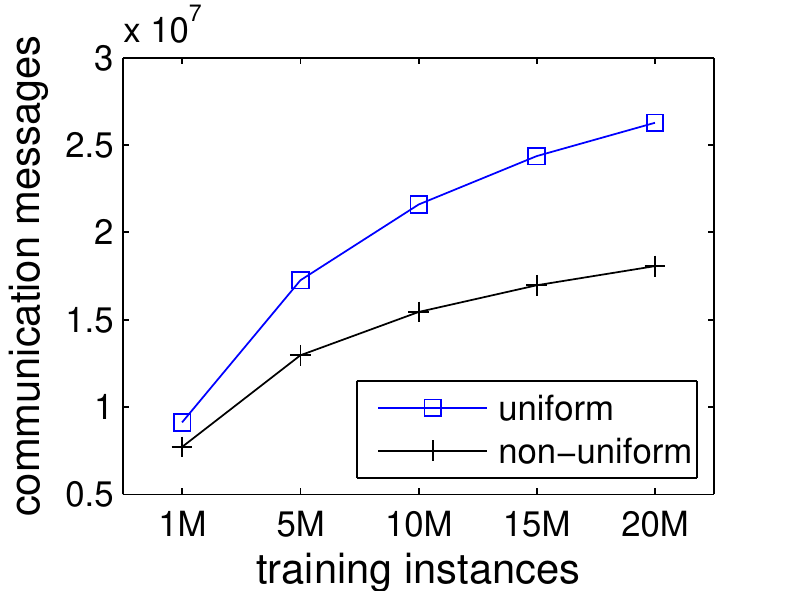}
\label{fig:comm-synthetic}
}
  \caption{Communication cost experiments.}
\end{figure}

\remove{
\begin{figure}[t]
  \centering \includegraphics[width=0.25\textwidth]{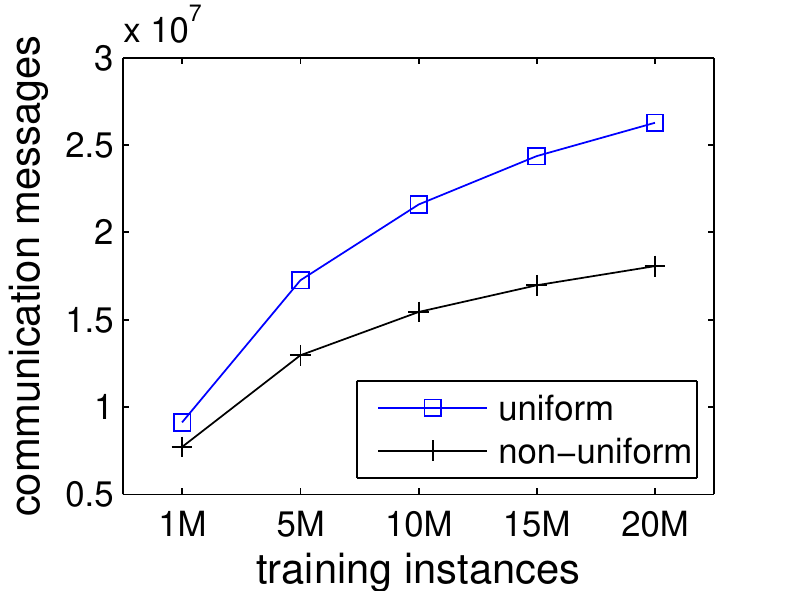}
\caption{Communication cost vs. number of training points, for the \newalarm dataset.}
\end{figure}
}

Last, Figure~\ref{fig:msg-vs-sites} plots communication cost against
the number of sites $k$, for the \alarm dataset, and shows that the
number of messages increases with $k$. 

\tpara{Communication Cost of \uniform versus \nonuniform:}
The results so far do not show a very large difference
in the communication cost of \uniform and \nonuniform.
The reason is that in the networks that we used, the cardinalities of
all random variables were quite similar.
In other words, for different $i \in [1,n]$, the $J_i$s in
Equation~\ref{eq:non-uniform-sol1} and~\ref{eq:non-uniform-sol2} have
similar values, and so did the $K_i$s.
This makes the approximation factors in \uniform and \nonuniform to be
quite similar. 
To study the communication efficiency of the non-uniform
approximate counter, we generated a semi-synthetic Bayesian network
\newalarm based on the \alarm network.
We keep the structure of the
graph, but randomly choose $6$ variables in the graph and set the size
of the universe for these values to $20$ (originally each variable
took between $2-4$ distinct values). The format of the synthetic
network can be downloaded at~\cite{synthetic}.
For this network, the communication cost of \nonuniform was about 35 percent
smaller than that of \uniform, in line with our expectations (Figure~\ref{fig:comm-synthetic}).

\begin{table}
\caption{Error Rate for Bayesian Classification, $50$K training instances}
{
\small
\footnotesize
\centering
\begin{tabular}[ht]{| l | r | r | r | r |}
\toprule
Dataset & \exact  & \baseline  & \uniform  & \nonuniform \\
\midrule
\alarm & 0.056 & 0.055 & 0.053 & 0.066 \\
\hepar & 0.191 & 0.187 & 0.198 & 0.212 \\
\link & 0.109 & 0.110 & 0.111 & 0.110 \\
\munin & 0.091 & 0.091 & 0.093 & 0.091 \\
\bottomrule
\end{tabular}\par
\label{table:error}
}
\end{table}

\begin{table}[t]
\caption{Communication cost to learn a Bayesian classifier}
{\footnotesize
\centering
\begin{tabular}
{| l | r | r | r | r |}
\toprule
Dataset & \exact  & \baseline  & \uniform  & \nonuniform \\
\midrule
\alarm & 3,700,000   &	406,721	 &	 323,710  &  322,639 \\
\hepar & 7,000,000   &	1,079,385	 &	 758,631 &	754,429 \\
\link & 72,400,000   &    29,781,937 &	8223133	 &	8,062,889 \\
\munin & 104,100,000 &	34,388,688 &	11,317,844 &	11,261,617 \\
\bottomrule
\end{tabular}\par
\label{table:communication}
}
\end{table}

\tpara{Classification: } Finally, we show results on learning a
Bayesian classifier for our data sets. 
For each testing instance, we first
generate the values for all the variables (using the underlying
model), then randomly select one variable to predict, given the values
of the remaining variables.
We compare the true value and predicted
value of the select variable and compute the error rate.
$1000$.
Prediction error and communication cost for $50$K examples and 1000
tests are shown
in Tables~\ref{table:error} and \ref{table:communication} respectively.

Overall, we first note that even the \exact algorithm has some
prediction error relative to the ground truth, due to the statistical
nature of the model. The error of the other algorithms, such as
\uniform and \nonuniform is very close to that of \exact, but their
communication cost is much smaller. For instance, \uniform and
\nonuniform send less than $1/9$th as many messages as \exact.

\section{Conclusion}
\label{sec:concl}
We presented new distributed streaming algorithmw 
to estimate the parameters of a \bayesnet in the distributed monitoring model.
Compared to approaches that maintain the exact MLE, 
our algorithms significantly reduce communication, while offering provable guarantees
on the estimates of joint probability. Our experiments show that
these algorithms indeed reduce communication and 
provide similar prediction errors as the MLE for 
estimation and classification tasks.

Some directions for future work include: (1)~to
adapt our analysis when there is a more skewed distribution across
different sites, (2)~to consider time-decay models which gives higher
weight to more recent stream instances, and (3)~to learn 
the underlying graph ``live" in an online fashion, as more data arrives.




\bibliographystyle{IEEEtran}
\bibliography{IEEEabrv,distBayes}



\end{document}